\documentclass[12pt]{article}
\usepackage{enumerate}
\usepackage{authblk}
\usepackage{xr}
\usepackage{url}            
\usepackage{amsfonts}       
\usepackage{nicefrac}       
\usepackage{microtype}      
\usepackage{natbib}
\usepackage{graphicx}
\usepackage{subcaption}
\usepackage{xcolor}
\usepackage{booktabs} 
\usepackage{amsmath}
\usepackage{amsthm}
\usepackage{amssymb}
\usepackage{mathtools}
\usepackage{algorithm,algorithmicx,algpseudocode}
\usepackage{thmtools}
\usepackage{thm-restate}
\usepackage{hyperref}
\hypersetup{hidelinks}
\usepackage{booktabs}
\usepackage{siunitx}
\usepackage{rotating}
\usepackage{pdflscape}
\usepackage{adjustbox}
\usepackage{array}
\usepackage{placeins}

\usepackage{multibib}       

\usepackage{graphicx}
\usepackage{mathptmx}
\usepackage{anyfontsize}
\usepackage{t1enc}
\usepackage{amsmath, amssymb, amsbsy}
\usepackage[utf8]{inputenc} 
\usepackage[T1]{fontenc}
\usepackage{amsmath,amsfonts,amssymb}
\usepackage{filecontents}

\makeatletter
\let\@internalcite\cite
\def\cite{\def\citeauthoryear##1##2{##1, ##2}\@internalcite}
\def\shortcite{\def\citeauthoryear##1{##2}\@internalcite}
\def\@biblabel#1{\def\citeauthoryear##1##2{##1, ##2}[#1]\hfill}
\makeatother

\usepackage{amsmath}

\usepackage{amsmath}
\usepackage{amsfonts}
\usepackage{amssymb}
\usepackage[a4paper]{geometry}
\usepackage{graphicx}

\usepackage{amsmath}

\usepackage{algorithm}

\usepackage{etoolbox}
\usepackage{microtype} 

\usepackage{authblk}
\usepackage{setspace}

\usepackage{bnumexpr}

\usepackage{algorithm}

\usepackage{hyperref}

\newcommand{\R}{\mathbb{R}}

\newcommand{\quant}{\mathsf{Q}}
\newcommand{\alg}{\mathcal{A}}

\newcommand{\Xcal}{\mathcal{X}}

\allowdisplaybreaks
\newcommand{\norm}[1]{\left\lVert#1\right\rVert}

\DeclareMathOperator*{\argminB}{argmin}   

\usepackage{booktabs}
\usepackage{siunitx}

\usepackage{xspace}

\usepackage[T1]{fontenc}
\usepackage{graphicx}

{\tiny }\usepackage{amsthm}
\newtheorem{example}{Example}
\newtheorem{theorem}{Theorem}

\newtheorem{proposition}[theorem]{Proposition}
\newtheorem{corollary}[theorem]{Corollary}

\newtheorem{definition}{Definition}
\newtheorem{remark}{Remark}
\newtheorem{assumption}{Assumption}

\usepackage{xr}
	
\newcommand{\blind}{1}

\usepackage{geometry}
\usepackage{amsmath}
\allowdisplaybreaks[4]
\usepackage{float}
\usepackage{mathrsfs}
\usepackage{amsfonts}   
\usepackage{xcolor}

\usepackage{times}
\usepackage{bm}
\usepackage{colonequals}
\usepackage{graphicx}
\usepackage{caption}
\usepackage{enumerate}
\usepackage{cmll}
\usepackage{geometry}
\usepackage{amsthm}
\usepackage{subcaption}
\usepackage{threeparttable}
\usepackage{booktabs}
\usepackage{multirow}
\usepackage{makecell}

\makeatletter

\makeatletter

\usepackage{titlesec}

\titlespacing*{\section}
  {0pt}    
  {10pt}   
  {5pt}    

\titlespacing*{\subsection}
  {0pt}
  {8pt}    
  {4pt}    

\usepackage{booktabs}
\usepackage{siunitx}
\newtheorem*{assumption*}{Assumption}
\newtheorem*{corolary*}{Corolary}

\usepackage{xcolor}

\makeatother

\usepackage{easyReview}

\usepackage{lipsum}

\newcommand\blfootnote[1]{%
	\begingroup
	\renewcommand\thefootnote{}\footnote{#1}%
	\addtocounter{footnote}{-1}%
	\endgroup
}

	\usepackage{xcolor}

\title{}
\begin{document}
	\pagenumbering{gobble}

	\def\spacingset#1{\renewcommand{\baselinestretch}%
		{#1}\normalsize} \spacingset{1}
	
	
	\if1\blind
	{
		\title{\bf Conformal and $k$NN Predictive Uncertainty Quantification Algorithms in Metric Spaces
		}
        
        \author[1,2,3]{Gábor Lugosi}
\author[4,5,6]{Marcos Matabuena}

\affil[1]{Department of Economics and Business, Pompeu Fabra University, Barcelona, Spain}
\affil[2]{ICREA, Pg. Llu\'is Companys 23, 08010 Barcelona, Spain}
\affil[3]{Barcelona Graduate School of Economics}
\affil[4]{Mohamed bin Zayed University of Artificial Intelligence}
\affil[5]{Department of Biostatistics, Harvard University, Boston, MA 02115, USA}
\affil[6]{Universidad de Santiago de Compostela}

\date{}
\maketitle
\blfootnote{$*$ Corresponding author. Email: Marcos.Matabuena@mbzuai.ac.ae; mmatabuena@hsph.harvard.edu}
} 
\fi

\if0\blind
{
\bigskip
\bigskip
\bigskip
\begin{center}
	{\bf \Large  Conformal and Non-Conformal Algorithms for Predictive uncertainty Quantification in metric spaces}
\end{center}
\medskip
} \fi
\bigskip
	\maketitle	
    \begin{abstract}
This paper introduces a framework for uncertainty quantification in regression models defined on metric spaces. Using a proposed notion of homoscedasticity, we define a conformal prediction algorithm that provides finite-sample marginal coverage guarantees and fast convergence rates to the oracle prediction region. For heteroscedastic settings, we introduce a $k$NN procedure that yields locally adaptive prediction radii in general metric spaces. Although this procedure does not provide the same finite-sample guarantees as the conformal algorithm, it is designed to improve local coverage calibration without imposing smoothing assumptions. Both procedures are compatible with a broad range of regression algorithms and scale to large datasets, allowing practitioners to use their preferred models and incorporate domain-specific knowledge. Building on the heteroscedastic $k$NN approach, we also develop a flexible sequential extension for metric-space-valued time series based on nearest-neighbor expert aggregation. We establish the consistency of the proposed estimators under minimal conditions. Finally, we illustrate the practical utility of our framework in personalized medicine applications involving random objects such as probability distributions and graph Laplacians.
\end{abstract}

	\section{Introduction}

	\subsection{Motivation and goal}	
\noindent	
The increasing use of statistical and machine-learning algorithms is transforming healthcare, finance, and digital markets \citep{Banerji2023,HAMMOURI2023,rodriguez2022contributions}. It has become more critical than ever to conduct uncertainty analysis to create trustworthy predictive models and validate their usefulness \citep{romano2019conformalized}. Data analysts tend to focus on pointwise estimation through the conditional mean between a response and a set of predictors, neglecting other crucial aspects of the conditional distribution between the involved random variables \citep{kneib2021rage}. In order to quantify the uncertainty of the point estimates, we need to estimate other characteristics of the conditional distribution beyond the conditional mean.

\noindent An innovative approach to tackle uncertainty quantification problems involves employing the conformal inference framework, first introduced by Gammerman, Vovk, and Vapnik  (see \citep{alex1998learning}). Building upon this foundation,  Vovk, Gammerman, and Shafer  have made significant contributions to the field \citep{vovk2005algorithmic}. Conformal inference allows one to construct prediction sets with non-asymptotic guarantees, setting it apart from other methods that only offer asymptotic guarantees. 

\noindent However, despite its attractive properties and advantages, conformal inference does have some limitations that include:

\begin{enumerate}
	\item \textbf{Computational Complexity:} The implementation of conformal inference methods can be computationally expensive, especially if data-splitting strategies to estimate the underlying regression model and prediction region (referred to as \emph{split-conformal} in the literature) are not considered \citep{vovk2018cross, solari2022multi}.
	
	\item \textbf{Conservative Intervals:} The method tends to produce conservative intervals if the regression function used is not well-calibrated with respect to the underlying model, or if the conformal inference method is misspecified \citep{lu2023data} (e.g., using a homoscedastic model when the underlying conditional distribution function between the random response variable $Y$ and the random predictor $X$ is governed by heteroscedastic random errors).
	
	\item \textbf{Limited Applicability:} Conformal inference may have limited applicability, particularly in settings with multivariate responses or when the observed data is extracted from complex survey designs where the exchangeability hypothesis is violated (see \cite{barber2022conformal}, who quantify the impact of this fact in terms of the total variation distance between distributions).
	
	\item \textbf{Asymptotic Guarantees:} In some special cases, other methods may provide stronger asymptotic guarantees (see \cite{gyofi2020nearest} for a discussion).
\end{enumerate}

\noindent Related to this, a fascinating and still under--explored question in the predictive--inference literature is how conformalizing an algorithm impacts its statistical efficiency. Vovk dubbed this general mathematical challenge the Burnaev--Wasserman program (see \cite{vovk2022efficiency}). Evidence from scalar--response settings and certain heteroscedastic settings suggests that conformal prediction can underperform in terms of conditional coverage calibration; for instance, this occurs when using bootstrap-predictive methods \cite{zhang2022regression}. This limitation becomes even more pronounced with high-dimensional responses or when the goal is to predict metric-space-valued responses. In such situations, the quality of the estimates tends to degrade more quickly. Consequently, the requirement of non--asymptotic coverage guarantees typically yields more conservative prediction regions. As a result, conformal prediction regions may substantially deviate from the population--prediction region, leading to a marked reduction in statistical efficiency for finite samples in terms of conditional coverage. Recently, several methods have been proposed to improve conditional calibration in conformal prediction (see for example \cite{Conformal_prediction_with, Sample-Conditional_Coverage}). Nevertheless, existing approaches introduce practical constraints: they are typically designed for linear spaces, may suffer in high-dimensional settings, require smoothness assumptions, and can be computationally demanding in their full conformal versions. A key open problem is therefore to develop methods that provide conditional guarantees in high-dimensional or non-Euclidean spaces under minimal smoothness conditions.

\noindent In the context of personalized and digital medicine, clinical outcomes can take on complex statistical forms as random objects, such as probability distributions or graphs \citep{rodriguez2022contributions}. For instance, in the case of a glucose time series, a modern approach to summarizing the glucose profiles involves using a distributional representation of the time series, such as their quantile functions. As prior research has shown, these new representations \citep{ghosal2023multivariate} can capture information about glucose homeostasis metabolism that traditional diabetes biomarkers cannot measure \citep{doi:10.1177/0962280221998064, matabuena2022kernel, 10.1093/biostatistics/kxab041}.

\noindent For all these reasons, there is a need for new predictive inference methods for regression models defined in metric spaces. To the best of our knowledge, the only existing approach that directly addresses predictive uncertainty quantification is that of \citet{zhou2024conformalinferencerandomobjects}, which was posted on arXiv concurrently with an earlier version of the present work \citep{lugosi2024uncertainty}. Zhou and Müller introduced a model--free conformal--inference procedure based on a conditional data depth--called the \emph{distance profile}. Their paper is a generalization to metric spaces of a scalar distributional split--conformal method of \cite{chernozhukov2021distributional}. However, the scalar procedure of \cite{chernozhukov2021distributional} relies on estimating the conditional distribution of several quantities and achieves fast rates only when that distribution can be well approximated by parametric or semiparametric models--an assumption rarely tenable in high--dimensional or metric--response problems. To alleviate this difficulty, the method of \cite{zhou2024conformalinferencerandomobjects} imposes  differentiability conditions that may limit applicability in discrete
or otherwise non-smooth settings. They also require data splits to preserve finite--sample conformal guarantees, and their theory is limited to a single univariate Euclidean covariate, leading to slow convergence in some settings.

\noindent In multivariate settings, the method continues to depend on kernel-smoothing estimators. However, to mitigate the curse of dimensionality, it aggregates all predictors into a single--index model, which increases the rigidity of the final model. Moreover, its reliance on smoothness assumptions limits applicability to mixed continuous and discrete covariates, as commonly encountered in physical--activity monitoring or survey--based medical studies.

\noindent Motivated by these shortcomings and by the loss of efficiency
 observed in existing conformal algorithms in high-dimensional  spaces, we propose a general uncertainty--quantification framework for regression on separable metric spaces (see \cite{AIHP_1948__10_4_215_0, petersen2019frechet, schotz2021frechet}). Noting that conformal prediction incurs a smaller efficiency loss under homoscedasticity, we introduce a new notion of homoscedasticity tailored to metric spaces and design conformal algorithms that estimate the regression function directly in that setting, attaining fast convergence rates. For heteroscedastic data, we propose a local $k$NN method that adapts naturally to non--linear geometric structure and scales to large datasets. Both procedures work with any regression algorithm, remain computationally efficient for large samples, and are shown to be consistent under some assumptions--often weaker than those required in the Euclidean literature, which is a special case of our framework.  The methods can also  operate under minimal conditions for dependent data, as we show in our variant for time series modeling.  We demonstrate the practical value of the methodology in personalized medicine applications where the responses are probability distributions or Laplacian graphs.

\noindent	Below, we introduce the notation and mathematical concepts to define the new uncertainty quantification framework.

	\subsection{Notation and problem definition}\label{sec:def}
	\noindent Let $(X,Y)\in \mathcal{X} \times \mathcal{Y}$ be a pair of random variables that play the role of the predictor and response variable in a regression model. Our predictive uncertainty quantification algorithms handle both predictors and responses in general metric spaces. However, for simplicity, we assume that	$\mathcal{X}= \mathbb{R}^{p}$ and $\mathcal{Y}$ is a  separable metric space equipped with a distance $d_1$.  We assume that there exists $y\in \mathcal{Y}$ such that $\mathbb{E}(d_{1}^{2}(Y,y)) < \infty$. The regression function $m$ is defined as
	
	\begin{equation}
	m(x)= \argminB_{y\in \mathcal{Y}} \mathbb{E}(d_{1}^{2}(Y,y)\mid X=x),
	\end{equation}
	
	\noindent	where $x\in \mathbb{R}^{p}$.  In other words, $m$ is the conditional Fréchet mean \citep{petersen2019frechet}. For simplicity, we assume that the minimum in $(1)$ is achieved for each $x$ and moreover, the conditional Fréchet mean of $Y$ given $X=x$, is unique. We note here that one may also consider the conditional Fréchet median obtained by $\argminB_{y\in \mathcal{Y}}\mathbb{E}(d_{1}(Y,y)\mid X=x).$ However, this is a special case of our setup obtained by replacing the metric $d_{1}$, by $\sqrt{d_{1}}$ (which is also a metric).

	\noindent	 Suppose that $\mathcal{Y}$ is also equipped with another distance $d_{2}$. We may have $d_{2}= d_{1}$, but in some cases it is convenient to work with two distances.
	
     \noindent For $y\in \mathcal{Y}$ and $r\geq 0$, denote by $\mathcal{B}(y,r):= \{z\in \mathcal{Y};\hspace{0.1cm} d_{2}(y,z)\leq r \}$, the closed ball of center $y\in \mathcal{Y}$, and radius $r$.

    \begin{remark}
$d_{1}$ denotes the metric used in the loss function used to obtain the regression function $m$, and $d_{2}$, possibly distinct from $d_{1}$, denotes the metric employed to derive the final prediction region. Allowing $d_{1}\neq d_{2}$ offers additional modeling flexibility: $d_{1}$ can capture the application--specific geometry of the response space, whereas selecting a simpler metric for $d_{2}$--for example, the supremum norm--often simplifies both the computation and the visualization of the resulting prediction sets, particularly in high-- or infinite‑-dimensional settings.

	\end{remark}

\noindent We now introduce the notion of \emph{homoscedasticity} for the random pair \((X,Y)\). The differentiation between homoscedasticity and heteroscedasticity is critical for our subsequent analysis. Based on the notion of \emph{homoscedasticity}, we can obtain  conformal prediction algorithms that are statistically efficient in general  metric spaces.

	\begin{definition} \label{def:homoc} We say that the distribution of $(X,Y)$ is homoscedastic with respect to the regression function $m$ if there exists a function $\phi: [0,\infty)\to [0,1]$ such that for all $x\in \mathbb{R}^{p}$ and $r\geq 0$, 
		$$\mathbb{P}(Y\in \mathcal{B}(m(x), r)\mid X=x)= \phi(r).$$ 
	\end{definition}

	\noindent To better motivate our definition of homoscedasticity, let us consider the following examples. The first example is the most natural and commonly encountered in the statistical literature. In this case, the notion of homoscedasticity aligns with the traditional definition.
	
	\noindent However, the second example is non-trivial and serves to illustrate that even in a space without a vector-space structure, it is still possible to have statistical models that fall under the homoscedastic regime according to our definition.
	
	\begin{example}
Suppose $\mathcal{X} = \mathbb{R}^p$, $\mathcal{Y} = \mathbb{R}^m$, and  
\[
d_1(\cdot,\cdot) = d_2(\cdot,\cdot) = \|\cdot - \cdot\|_2,
\]
where $\|\cdot\|_2$ denotes the Euclidean norm on $\mathbb{R}^m$.  
Consider the regression model  
\begin{equation}
Y = f(X) + \epsilon,
\end{equation}
where $\epsilon$ is a random vector taking values in $\mathbb{R}^m$, independent of $X$, such that  $\mathbb{E}[\epsilon] = 0$. Then,
\[
\mathbb{P}\bigl( d_2(Y, m(x)) \leq r \mid X = x \bigr)
= \mathbb{P}\bigl( \|\epsilon\|_2 \leq r \mid X = x \bigr)
= \phi(r),
\]
where the second equality follows from the independence of $\epsilon$ and $X$.  
Therefore, the model is homoscedastic.
\end{example}

\begin{example}\label{ejempldos}

	\noindent Consider $\mathcal{X} = \mathbb{R}^{p}$ and $\mathcal{Y} = \mathcal{W}_{2}(\mathbb{R})$, where $\mathcal{W}_{2}(\mathbb{R})$ denotes the 2-Wasserstein space (see \cite{ panaretos2020invitation}). More specifically,  we define in view of our generative example $\mathcal{W}_{2}(\mathbb{R}) = \{ F \in \Pi(\mathbb{R})\}$, where $\Pi(\mathbb{R})$ is the set of distribution functions over $\mathbb{R}$ with a finite number of discontinuities.	A natural metric for $\mathcal{W}_{2}(\mathbb{R})$ is $d_{\mathcal{W}_{2}}(F,G)$, is defined as 
	\begin{equation}
	d_{\mathcal{W}_{2}}(F,G) = \sqrt{\int_{0}^{1} (Q_{F}(t) - Q_{G}(t))^{2} dt}.
	\end{equation}

\noindent	Here, $Q_{F}$ and $Q_{G}$ denote the quantile functions of the distribution functions $F$ and $G$, respectively. Define the discrete-time stochastic process $Z(\cdot)$ as
\begin{equation*}
Z(j) = g(X) + \epsilon(j), \quad j = 1, \dots, n,
\end{equation*}
\noindent where $X$ is a random vector taking values in $\mathbb{R}^{p}$, $g: \mathbb{R}^{p} \to \mathbb{R}$ is a continuous function, and for each $t = 1, \dots, n$, $\epsilon(t) \sim \mathcal{N}(0, \sigma^{2}_{\epsilon})$, with $\epsilon(j) \perp \epsilon(j')$ for $j \neq j'$, where $\perp$ denotes the statistical independence between two random variables, and $\epsilon(1),\dots, \epsilon(n)$ are independent of $X$.
	
\noindent	Then, we define the element $F_{n} \in \mathcal{W}_{2}(\mathbb{R})$ as
	\begin{equation}
	F_{n}(t) = \frac{1}{n} \sum_{j=1}^{n} \mathbb{I}\{Z(j) \leq t\}.
	\end{equation}

	
	
	

\noindent	
For $\rho\in(0,1]$, define
\[
Q_n(\rho)=\inf\{t:F_n(t)\geq \rho\}.
\]
If $\epsilon_{(1)}\leq\cdots\leq\epsilon_{(n)}$ are the order statistics of
$\epsilon(1),\ldots,\epsilon(n)$, then
\[
Q_n(\rho)
=
g(X)+\epsilon_{(\lceil n\rho\rceil)}.
\] The expectation of the quantile function can be easily derived as:
\begin{equation}
\mathbb{E}(Q_{n}(\rho)\mid X) = g(X) + \sigma_{\epsilon} h_{n}(\rho),
\end{equation}

\noindent where $h_{n}(\cdot)$ is a function that depends on the sample size $n$.  

\noindent By the definition of $d^{2}_{\mathcal{W}_{2}}(\cdot,\cdot)$, $\left[m(X)\right]^{-1}(\rho)= \mathbb{E}\left[Q_{n}(\rho)|X\right]$. To see this, we note that

\begin{equation}
m(x)= \argminB_{y\in \mathcal{W}_{2}(\mathbb{R})} \mathbb{E}(d^{2}_{\mathcal{W}_{2}}(F_{n},y)\mid X=x)=  \argminB_{y\in \mathcal{W}_{2}(\mathbb{R})} \mathbb{E}\left( \int_{0}^{1}\left(Q_{n}(t)-y^{-1}(t))^{2}dt\right)\mid X=x\right) = \mathbb{E}\left[Q_{n}\mid X=x\right]^{-1},
\end{equation}

\noindent where, in the final step, we use the property that the mean operator minimizes the square distance in Euclidean space. Then, the squared Wasserstein distance between $F_n$ and its conditional
Fr\'echet mean $m(X)$ is
\[
d_{\mathcal{W}_2}^2(F_n,m(X))
=
\int_0^1
\left[
\epsilon_{(\lceil n\rho\rceil)}
-
\sigma_\epsilon h_n(\rho)
\right]^2
\,\mathrm{d}\rho,
\]
which does not depend on $X$.
\end{example}

	\noindent We say that the distribution of $(X,Y)$ is heteroscedastic with respect to the regression function $m$  when
	$$\mathbb{P}(Y\in  \mathcal{B}(m(x),r)|X=x)$$ 
	\noindent	 may depend on  $x$.
		
	\noindent	Given a random observation $X$, and confidence level $\alpha\in [0,1]$, our  goal is  to estimate a prediction region $C^{\alpha}(X)\subset \mathcal{Y}$ of the response variable $Y$  that  contains the response variable with probability at least $1-\alpha$. First, we  introduce the population version of the  prediction region.

	\begin{definition}
	\noindent Define the oracle-prediction region as   
		\begin{equation} 
		C^{\alpha}(x):= \mathcal{B}(m(x),r(x)),
		\end{equation}
		\noindent where $r(x)$ is the smallest number $r$ such that $\mathbb{P}(Y\in \mathcal{B}(m(x),r)|X=x)\geq 1-\alpha$. 
	\end{definition}

\noindent In this work, we define the prediction region as a ball in the output space, which naturally serves as a level set in a general metric space. We acknowledge that this choice restricts level sets to isotropic regions (with respect to the geometry induced by the metric $d_2$).

	\noindent In order to estimate	$ C^{\alpha}(x)$, except in the specific section for dependent data, we suppose that we observe a random sample $\mathcal{D}_{n}=\{ (X_i,Y_i)\in \mathcal{X}\times \mathcal{Y}: i\in [n]:= \{1,2,\dots,n\}\}$ containing independent, identically distributed (i.i.d.) pairs drawn from the same distribution as $(X,Y)$.  In the rest of this paper, we  split  $\mathcal{D}_{n}$ into two disjoint subsets,  $\mathcal{D}_{n}=\mathcal{D}_{train} \cup \mathcal{D}_{test}$, in order to estimate the center and radius of the ball with two independent random samples.   We denote the set of indexes $[S_1]:= \{ i\in [n]:  \left(X_i,Y_i\right)\in   \mathcal{D}_{train} \}$, $[S_2]:= \{ i\in [n]:  (X_i,Y_i)\in   \mathcal{D}_{test} \}$, and we denote  $|\mathcal{D}_{train}|= n_1$ and $|\mathcal{D}_{test}|= n_2= n-n_1$.	
	
	\noindent	We propose estimators of the oracle-prediction region $C^{\alpha}(x)$ of the form $\widetilde{C}^{\alpha}(x)=\mathcal{B}(\widetilde{m}(x),\widetilde{r}_{\alpha}(x))$, where $\widetilde{m}(x)$ and $\widetilde{r}_{\alpha}(x)$ are estimators of $m(x)$ and $r_{\alpha}(x)$.


	\begin{definition}
	Consider a measurable mapping $C: \mathcal{X} \to 2^{\mathcal{Y}}$. We define the error of $C$ with respect to the  oracle-prediction region as
		
		\begin{equation*}
		\varepsilon(C,x)= \mathbb{P}(Y\in C(X)\triangle C^{\alpha}(X)\mid X=x), 
		\end{equation*}

		\noindent where $C(X)\triangle C^{\alpha}(X)$ denotes the symmetric difference between the sets $C(X)$ and $C^{\alpha}(X)$.

		\noindent	If $C= \widetilde{C}^{\alpha}$, is constructed from $\mathcal{D}_{n}$, then  
		
		\begin{equation*}
		\varepsilon(\widetilde{C},x)= \mathbb{P}(Y\in \widetilde{C}^{\alpha}(X)\triangle C^{\alpha}(X)\mid X=x,\mathcal{D}_{n}).	
		\end{equation*} 
		
		\noindent We are interested in the integrated error

		\begin{equation*}
		\mathbb{E}(\varepsilon(\widetilde{C}^{\alpha},X)|\mathcal{D}_{n}).
		\end{equation*}

		\noindent We say that $\widetilde{C}^{\alpha}$ is consistent if

		\begin{equation}
		\lim_{n\to \infty} \mathbb{E}(\varepsilon(\widetilde{C}^{\alpha},X)|\mathcal{D}_{n})\to 0  \text{ in probability. }
		\end{equation}
	\end{definition}

	\subsection{Contributions}


This paper develops a general framework for uncertainty quantification for random objects in metric spaces. The framework constructs prediction regions without imposing smoothness assumptions, aims to achieve good conditional calibration \cite{Conformal_prediction_with}, and provides finite-sample guarantees in some special cases.   Now, we summarize  the main contributions of this paper:

    \begin{enumerate}
\item In the Euclidean setting, the performance and applicability of prediction methods are strongly influenced by the signal--to-noise ratio of the regression model $m$. We propose a novel notion of homoscedasticity. Based on this:
		
        We propose three uncertainty quantification algorithms: one for the homoscedastic case, another for the general case of independent data, including heteroscedasticity, and a third for temporally correlated time-series sequences defined on metric spaces.
		
        \begin{enumerate}
			\item  Homoscedastic case: We propose a natural generalization of split conformal inference techniques \citep{vovk2018cross, solari2022multi} in the context of metric space regression algorithms. We obtain non-asymptotic guarantees of marginal coverage of the type $\mathbb{P}(Y\in\widetilde{C}^{\alpha}(X))\geq 1-\alpha$. In this setting, we also show that $\widetilde{C}$ is consistent,	under mild assumptions, which only require the consistency of the conditional Fréchet mean estimator, i.e., $\mathbb{E}(d_{2}(\widetilde{m}(X),m(X))| \mathcal{D}_{train})\to 0$ in probability as $|\mathcal{D}_{train}|\to \infty$. We also establish  bounds for the convergence rate of the derived prediction region.
		
			\item  Heteroscedastic case: We introduce a local $k$NN method for predictive inference in metric spaces. Although the resulting prediction regions do not enjoy finite-sample marginal coverage guarantees, we establish their consistency and derive finite-sample upper bounds on their expected error. Because the performance of the method is sensitive to the neighborhood size $k$, we propose a data-driven procedure for selecting it. The resulting algorithm provides good conditional coverage behavior while producing statistically efficient prediction regions across a range of practical settings.

\item Building on the previous $k$NN approach, we propose a novel expert-aggregation method for metric-space-valued time series, motivated by online learning theory \citep{cesa2006prediction,biau2011sequential}. Under stationarity and ergodicity, we establish the statistical consistency of the proposed procedure under mild regularity conditions. In addition, the expert-based approach is flexible in practice and avoids the smoothness assumptions typically required by nonlinear time-series methods.

			\item The prediction regions are the generalization of the concept of conditional quantiles for metric space responses. The prediction regions are defined by a center function $m$ and a specific radius $r$ corresponding to a chosen probability level, $\alpha\in [0,1]$. This extension builds upon recent advances in the theory of unconditional quantiles for metric spaces \cite{liu2022quantiles} and provides a natural generalization of the existing concepts.
			
		\end{enumerate}

Arguably, the proposed uncertainty quantification framework, built around
metric balls, offers a natural and broadly applicable alternative to existing
depth-profile methodologies \citep{zhou2024conformalinferencerandomobjects}.
In particular, it extends naturally to settings involving causal inference,
non-Euclidean predictors, and high-dimensional Euclidean spaces.

Computationally, after estimating the regression function \(m\), the proposed
algorithms can process millions of observations in a few seconds. By contrast,
depth-profile methods require repeated local polynomial smoothing and are
practically limited to moderately sized datasets, with reported implementations
handling samples of size around \(n \approx 4{,}000\) only after several hours
of computation. This makes them less suitable for large--scale applications.
Moreover, these methods do not provide fully data-driven procedures for
selecting the smoothing bandwidth \(h\).

Our framework also operates under substantially weaker assumptions. Whereas depth-profile approaches require regression estimators with multiple orders of differentiability and introduce additional smoothing steps, our method avoids these requirements while maintaining good conditional-coverage calibration. In this sense, the proposed framework extends recent ideas on local conditional calibration in conformal prediction \citep{Conformal_prediction_with,Sample-Conditional_Coverage} to metric-space settings through a data-adaptive class of metric neighborhoods, without relying on smoothing-based arguments.


	\item We evaluate our framework on real--world biomedical datasets, demonstrating its ability to handle complex, high--dimensional non--Euclidean structures. These experiments highlight its suitability for modern digital--health and personalized-medicine applications. The different  applications involve probability distributions with the $2$--Wasserstein metric, fMRI data with graph Laplacians, and multivariate Euclidean data.


	\end{enumerate}

\subsection{General view of existing uncertainty quantification methods}\label{sec:summary}

\noindent In recent years, uncertainty quantification became an  active research area  \citep{geisser2017predictive, politis2015model}. The impact of uncertainty quantification on data-driven systems has led to a remarkable surge of interest in both applied and theoretical domains. These works address fundamental questions of uncertainty quantification in statistical science and beyond, such as in biomedicine \citep{HAMMOURI2023, Banerji2023}.

\noindent Geisser's pioneering book develops a mathematical theory of prediction inference \cite{geisser2017predictive}. Building upon Geisser's foundations, Politis presented a comprehensive methodology that effectively applies resampling techniques \citep{politis2015model}. Additionally, the  book of Vovk, Gammerman,  and Shafer \cite{vovk2005algorithmic} (see  recent reviews \citep{angelopoulos2021gentle,fontana2023conformal}) 
has been influential.

\noindent One of the most widely used and robust frameworks for quantifying uncertainty in statistical and machine learning models is conformal inference \citep{shafer2008tutorial}. The central idea of conformal inference is rooted in the concept of exchangeability \citep{kuchibhotla2020exchangeability}. For simplicity, we assume that the data are independent and identically distributed (i.i.d.).
\noindent Now, we present a general overview of conformal inference methods for regression models with scalar responses. Consider the sequence $\mathcal{D}_{n} = \{(X_i, Y_{i})\}^{n}_{i=1}$ of i.i.d. random variables. Given a new i.i.d. random pair $(X, Y)$, conformal prediction, as introduced by \cite{vovk2005algorithmic}, provides a family of algorithms for constructing prediction intervals independently of the regression algorithm used.

\noindent Fix any  regression algorithm 
\[\alg: \ \cup_{n\geq 0} \left(\Xcal\times \R\right)^n \ \rightarrow \
  \left\{\textnormal{measurable functions $\widetilde{m}: \Xcal\rightarrow\R$}\right\}, 
\] 
\noindent which maps a data set containing any number of pairs $(X_i,Y_i)$, to a fitted
regression function $\widetilde{m}$. The algorithm $\alg$ is required to treat
data points symmetrically, i.e.,

\begin{equation}\label{eqn:alg_symmetric}
\alg\big((x_{\pi(1)},y_{\pi(1)}),\dots,(x_{\pi(n)},y_{\pi(n)})\big) =
\alg\big((x_1,y_1),\dots,(x_n,y_n)\big)
\end{equation}
for all $n \geq 1$, all permutations $\pi$ on $[n]=\{1,\dots,n\}$, and all
$\{(x_i,y_i)\}_{i=1}^{n}$. Next, for each $y\in\R$, let \[
\widetilde{m}^{y} = \alg\big((X_1,Y_1),\dots,(X_n,Y_n),(X,y)\big)
\]
denote the trained model, fitted to the training data together with a
hypothesized test point $(X,y)$, and let  
\begin{equation}\label{eqn:R_y_i}
R^y_i = 
\begin{cases}|Y_i - \widetilde{m}^{y}(X_i)|, & i=1,\dots,n,\\ 
|y-\widetilde{m}^{y}(X)|, & i=n+1.
\end{cases}
\end{equation}
The prediction interval for $X$ is then defined as 
\begin{equation}\label{eqn:def_fullCP}
\widetilde{C}^{\alpha}(X;\mathcal{D}_{n}) =\left\{y \in\R \ : \ R^y_{n+1}\leq
  \quant_{1-\alpha}\left(\sum_{i=1}^{n+1} \tfrac{1}{n+1} \cdot \delta_{R^y_i}
  \right)\right\}, 
\end{equation}

\noindent  where $\quant_{1-\alpha}\left(\sum_{i=1}^{n+1} \tfrac{1}{n+1} \cdot \delta_{R^y_i}
  \right)$ denotes the quantile  of order $1-\alpha$ of the empirical distribution $\sum_{i=1}^{n+1} \tfrac{1}{n+1} \cdot \delta_{R^y_i}$.

\noindent The full conformal method is known to guarantee distribution-free
predictive coverage at the target level $1-\alpha$: 

\begin{theorem}[Full conformal prediction]\label{thm:background_fullCP} 
\cite{vovk2005algorithmic} If the data points $(X_1,Y_1),\dots,(X_n,Y_n),(X,Y)$ are i.i.d.\ (or
more generally, exchangeable), and the algorithm $\alg$ treats the input data
points symmetrically as in~\eqref{eqn:alg_symmetric}, then the full conformal
prediction set defined in~\eqref{eqn:def_fullCP} satisfies 
\[
\mathbb{P}(Y\in \widetilde{C}^{\alpha}(X; \mathcal{D}_{n})) \geq 1-\alpha.
\]
\end{theorem}
\noindent A similar result holds  for split conformal methods, which involve estimating the regression function on $\mathcal{D}_{\text{train}}$ and the quantile on $\mathcal{D}_{\text{test}}$.

\noindent Conformal inference (both split and full) was
initially proposed in terms of ``nonconformity scores''
\smash{$\widetilde{S}(X_i,Y_i)$}, where \smash{$\widetilde{S}$} is a fitted function
that measures the extent to which a data point $(X_i,Y_i)$ is unusual relative to
a training data set $\mathcal{D}_{train}$. We have only presented the
most commonly used nonconformity score, which is the residual from the fitted
model.
\begin{equation}\label{eqn:standard_nonconformity_score}
\widehat{S}(X_i,Y_i) := |Y_i - \widetilde{m}(X_i)|,
\end{equation}
\noindent where $\widetilde{m}$ is the estimator of the regression function, estimated using random elements from the first split, $\mathcal{D}_{train}$.

\noindent 

\noindent \cite{foygel2021limits} investigated the case of the probability of conditional coverage with respect to a specific point $x \in \mathcal{X}$, which means that

\begin{equation}\label{eqn:cond}
 \mathbb{P}(Y \in \widetilde{C}^{\alpha}(X; \mathcal{D}_{n}) \mid X=x)\geq 1-\alpha.
\end{equation}

\noindent The standard finite-sample conformal guarantee is marginal (averaging over training and future data). \cite{vovk2012conditional} gave a PAC-type training-conditional guarantee for inductive conformal predictors: with calibration size $n_{\mathrm{cal}}$, for any $\delta\in(0,1)$,
\[
\mathbb{P}_{\mathcal D_n}\!\left[ \mathbb{P}\!\left\{ Y\notin\widetilde C^\alpha(X;\mathcal D_n) \mid \mathcal D_n \right\} \leq \alpha + \sqrt{\frac{\log(1/\delta)}{2n_{\mathrm{cal}}}} \right] \geq 1-\delta.
\]
Equivalently, with probability $\geq 1-\delta$ over the training/calibration data, conditional coverage is at least $1-\alpha - \sqrt{\log(1/\delta)/(2n_{\mathrm{cal}})}$. Thus, inductive conformal prediction provides approximate PAC-type training-conditional validity, not an exact conditional coverage guarantee for each realized training sample.

\noindent In the general context of i.i.d. and beyond the conformal inference framework, within the asymptotic regime, \cite{gyofi2020nearest} presented a kNN-based uncertainty quantification algorithm conditioned on both: i) a random sample $\mathcal{D}_{n}$ and ii) covariate characteristics. This algorithm possesses the property that 

\begin{equation*}
\lim_{n\to \infty} \mathbb{P}(Y \in \widetilde{C}^{\alpha}(X; \mathcal{D}_{n})| \mathcal{D}_{n}, X=x)= 1-\alpha \text{ in probability}.
\end{equation*}

\noindent The work of \cite{gyofi2020nearest} is particularly relevant as it provides optimal non-parametric rates for $k$NN-based uncertainty quantification. Several authors  established non-asymptotic guarantees for the coverage outlined in equation \eqref{eqn:cond}. However, in general, this is impossible without incorporating strong assumptions about the joint distribution of $(X, Y)$ or possessing exact knowledge of the distribution \citep{foygel2021limits, Conformal_prediction_with}.

\noindent \cite{lei2014distribution} show that forcing a prediction algorithm to satisfy the condition in \eqref{eqn:cond} can cause the Lebesgue measure of the estimated prediction region to become unbounded $\widetilde{C}^{\alpha}(X; \mathcal{D}_{n})$. More specifically, they prove that, for any continuous probability law $P,$

\begin{equation*}
\lim_{n \to \infty} \mathbb{E}[\text{Leb}(\widetilde{C}^{\alpha}(X; \mathcal{D}_{n}))] = \infty,
\end{equation*}

\noindent where $\text{Leb}$ denotes the Lebesgue measure.

\paragraph{Existing methods for random objects in metric spaces}

The only existing uncertainty quantification algorithm with a comparable goal is the conformal approach proposed by \citet{zhou2024conformalinferencerandomobjects}. Their procedure constructs prediction regions for responses that reside in a general metric space, but it is limited to scalar predictors (or, in the multivariate setting, to a single-index projection step). The method relies on a two-stage local polynomial estimation of conditional distribution functionals, and its precision is sensitive to the bandwidth parameter~$h$. As a result, it does not extend smoothly to non-Euclidean predictors, non--differentiable models, or high--dimensional covariate spaces. The core steps of the algorithm proposed by \citep{zhou2024conformalinferencerandomobjects} are outlined in Algorithm~\ref{alg:confc}. To the best of our knowledge, general methods for uncertainty quantification in metric spaces under dependent data remain largely undeveloped. Our work contributes to this literature by proposing a flexible framework that provides conditional calibration guarantees for random objects in metric spaces while avoiding smoothing assumptions \cite{Conformal_prediction_with}.

Existing approaches are typically not designed to handle non-Euclidean responses or predictors and often rely on additional regularity conditions. In contrast, our kNN-based implementation is computationally efficient, broadly applicable, and operates under minimal assumptions. Moreover, numerical experiments demonstrate that it is highly competitive in practice.

\begin{algorithm}[ht]
  \caption{{Conformal prediction for random responses in metric spaces and scalar predictors \citep{zhou2024conformalinferencerandomobjects}}}
  \label{alg:zhou}
  \begin{algorithmic}[1]
    \Require Confidence level $\alpha\in(0,1)$, kernel $K_h$, split fraction $\tau\in(0,1)$.
    \State \textbf{Data split.} Randomly split $\mathcal{D}_{n}=\{(X_i,Y_i)\}_{i=1}^n\subset \mathbb{R}\times \mathcal{Y}$ into
           training set $\mathcal D_{\mathrm{train}}$ of size $\lfloor\tau n\rfloor$
           and calibration set $\mathcal D_{\mathrm{test}}$.
    \State \textbf{Conditional distance profiles.}  For each $x\in\mathcal X$ and
           $\omega\in\mathcal \mathcal{Y}$, estimate
      \[
        \widetilde F_{\omega,x}(t)=
        \frac{\sum_{i\in\mathcal D_{\mathrm{train}}}K_h(x-X_i)\,\mathbb I\{d(\omega,Y_i)\le t\}}
             {\sum_{i\in\mathcal D_{\mathrm{train}}}K_h(x-X_i)}, \qquad t\ge0.
      \]

\State \textbf{Average transport cost.}
\[
\widetilde C(\omega\mid x)
=
\mathbb E\left[
\left.
\int_0^\infty
\left|
\widetilde F_{\omega,x}(t)
-
\widetilde F_{Y,x}(t)
\right|
\,dt
\;\right|\;
X=x
\right].
\]

    \State \textbf{Conditional profile score.}
      \[
        \widetilde S(z\mid x)=
        \frac{\sum_{i\in\mathcal D_{\mathrm{train}}}K_h(x-X_i)\, \mathbb I\{\widetilde C(Y_i\mid X_i)\le z\}}
             {\sum_{i\in\mathcal D_{\mathrm{train}}}K_h(x-X_i)}, \qquad z\ge0.
      \]
    \State \textbf{Calibration quantile.}  For each $(X_i,Y_i)\in\mathcal D_{\mathrm{test}}$,
           set $\widetilde S_i:=\widetilde S\bigl(\widetilde C(Y_i\mid X_i)\mid X_i\bigr)$ and define
      \[
        \widetilde Q_{1-\alpha} = 
        \operatorname{Quantile}_{\, (1-\alpha)\bigl(1+\frac{1}{|\mathcal D_{\mathrm{test}}|}\bigr)}
        \bigl\{\widetilde S_i : i\in\mathcal D_{\mathrm{test}}\bigr\}.
      \]
    \State \textbf{Scoring a new covariate.} For a new random predictor $X_{n+1}\in \mathbb{R}$ and any $y\in\mathcal{Y}$,
           compute $\widetilde C(y\mid X_{n+1})$ and
           $\widetilde S\bigl(\widetilde C(y\mid X_{n+1})\mid X_{n+1}\bigr)$.
    \State \textbf{Prediction set.}
      \[
        \widetilde C_\alpha(X_{n+1})=
        \bigl\{\,y\in \mathcal{Y} :
                \widetilde S\bigl(\widetilde C(y\mid X_{n+1})\mid X_{n+1}\bigr)\le\widetilde Q_{1-\alpha}
        \bigr\}.
      \]
  \end{algorithmic}
  \label{alg:confc}
\end{algorithm}

	\section{Mathematical models}

\noindent In this section, we present a framework for uncertainty quantification, based on the previous definition of the oracle prediction region (see Section \ref{sec:def}):

	\begin{equation}
	C^{\alpha}(x):= \mathcal{B}(m(x),r(x)).
	\end{equation}
	
	\noindent Recall that $r(x)$ is the smallest number $r$ such that $\mathbb{P}(Y\in \mathcal{B}(m(x),r)\mid X=x)\geq 1-\alpha$.
	
	\noindent The initial phase of our uncertainty quantification methodology entails the estimation of the Fréchet regression function $m:\mathcal{X}\to \mathcal{Y}$.

\noindent In the second step, we determine the radius that governs the number of points contained within the prediction regions. This selection is made to ensure that the set includes at least a $(1-\alpha)$ proportion of points from the dataset.

\noindent We emphasize that our method establishes, for different confidence levels $\alpha \in [0,1]$, different level sets for the prediction regions problem. This approach allows us to associate the prediction regions with the notion of conditioned quantiles for metric space responses, generalizing the marginal approaches from \cite{liu2022quantiles}.

	\subsection{Homoscedastic case}\label{sec:ho}

    Algorithm \ref{alg:metd1} outlines the core steps of our uncertainty quantification method for the homoscedastic case. The data are divided into two independent subsets: a training set \(\mathcal{D}_{\mathrm{train}}\), used to estimate the center function \(m\), and a calibration set \(\mathcal{D}_{\mathrm{cal}}\), used to estimate the radius of the prediction set.

\noindent In the homoscedastic case, the conditional distribution of \(d_{2}(Y,m(x))\) does not depend on the point \(x \in \mathcal{X}\). Therefore, after estimating \(m\) by \(\widetilde{m}\) using \(\mathcal{D}_{\mathrm{train}}\), it is natural to construct a single global radius from the empirical distribution of the calibration scores
\[
\left\{
d_{2}(Y_i,\widetilde{m}(X_i)) :
(X_i,Y_i)\in \mathcal{D}_{\mathrm{test}}
\right\}.
\]
To obtain a finite-sample conformal guarantee, the radius is chosen as the empirical quantile at the adjusted level
\[
(1-\alpha)\left(1+\frac{1}{|\mathcal{D}_{\mathrm{test}}|}\right).
\]
This yields a conformalized version of the algorithm satisfying the marginal non-asymptotic guarantee
\[
\mathbb{P}\left(
Y\in \widetilde{C}^{\alpha}(X;\mathcal{D}_{n})
\right)
\geq 1-\alpha.
\]

\noindent The idea of estimating different quantities of the model on different data splits is standard in the literature on conformal inference. This approach increases computational feasibility, provides non-asymptotic guarantees, and facilitates theoretical asymptotic analysis. However, it may also lead to some loss of statistical efficiency, since model quantities are estimated on different subsamples. These methods, known as split conformal methods, have been extensively studied; see, for example, \cite{vovk2018cross, solari2022multi}.

\begin{algorithm}[ht!]
\caption{Uncertainty quantification algorithm; homoscedastic case}
\label{alg:metd1}
\begin{enumerate}
    \item Estimate the function \(m(\cdot)\) by \(\widetilde{m}(\cdot)\) using the training sample \(\mathcal{D}_{\mathrm{train}}\).

    \item For each \((X_i,Y_i)\in \mathcal{D}_{\mathrm{test}}\), evaluate \(\widetilde{m}(X_i)\) and define the calibration score
    \[
    \widetilde{r}_i
    =
    d_{2}(Y_i,\widetilde{m}(X_i)).
    \]

    \item Let \(n_{2} = |\mathcal{D}_{\mathrm{test}}|\). Estimate the empirical distribution of the calibration scores by
    \[
    \widetilde{G}^{*}(t)
    =
    \frac{1}{n_{2}}
    \sum_{(X_i,Y_i)\in \mathcal{D}_{\mathrm{test}}}
    \mathbb{I}\{\widetilde{r}_i \leq t\}.
    \]
    Denote by \(\widetilde{q}_{\alpha}\) the empirical quantile at level
    \[
    min\left(1,(1-\alpha)\left(1+\frac{1}{n_{2}}\right)\right),
    \]
    that is,
    \[
    \widetilde{q}_{1-\alpha}
    =
    \inf
    \left\{
    t :
    \widetilde{G}^{*}(t)
    \geq
    (1-\alpha)\left(1+\frac{1}{n_{\mathrm{2}}}\right)
    \right\}.
    \]
   
    \item Return the prediction set
    \[
    \widetilde{C}^{\alpha}(x;\mathcal{D}_{n})
    =
    \mathcal{B}
    \left(
    \widetilde{m}(x),
    \widetilde{q}_{1-\alpha}
    \right).
    \]
\end{enumerate}
\end{algorithm}

\noindent We have the following marginal finite-sample guarantee.

\begin{proposition}
\label{th:finite}
Assume that the random elements \(\{(X_i,Y_i)\}_{i=1}^{n}\) and the test point \((X,Y)\) are i.i.d. Let \(\widetilde{m}\) be any estimator constructed using only \(\mathcal{D}_{\mathrm{train}}\). Then Algorithm \ref{alg:metd1} satisfies
\[
\mathbb{P}
\left(
Y \in \widetilde{C}^{\alpha}(X;\mathcal{D}_{n})
\right)
\geq
1-\alpha.
\]
\end{proposition}

\begin{remark}
\noindent Proposition \ref{th:finite} holds in both homoscedastic and heteroscedastic cases. However, when we apply Algorithm \ref{alg:metd1} to a heteroscedastic model, the resulting conditional coverage can substantially deviate from the oracle prediction region. Under such circumstances, the algorithm may not be consistent.

\end{remark}

	\begin{remark}
In discrete response spaces, ties among the calibration scores
\[
\widetilde r_j = d_2(Y_j,\widetilde m(X_j))
\]
may occur with positive probability. These ties do not invalidate the finite-sample lower coverage guarantee of Algorithm~\ref{alg:metd1}, but they may make the resulting prediction region more conservative. To obtain the usual near-exact coverage bound, one may use randomized tie-breaking, for example by augmenting each calibration score with an independent continuous auxiliary random variable. The resulting randomized scores are distinct almost surely. This strategy is discussed, for example, by \citet[Definition~1]{kuchibhotla2020exchangeability} and by \citet{cauchois2021knowing} in the context of discrete structures. Under randomized tie-breaking,
\[
1-\alpha \le \mathbb{P}( Y\in\widetilde C^\alpha(X;\mathcal D_n)) \le 1-\alpha+\frac{1}{n_{2}+1}.
\]
\end{remark}

	\subsubsection{Statistical consistency}

	Next we introduce the technical assumptions that ensure the asymptotic consistency of Algorithm \ref{alg:metd1}.
	
	\begin{assumption}
		Suppose that the following hold:
		\begin{enumerate}\label{1Smet}
			
            \item The distribution of $(X,Y)$ is homoscedastic with respect to $m$.
            \item $n_1,n_2\to \infty.$
			\item $\widetilde{m}$ is a consistent estimator in the sense that $\lim_{n_{1}\to \infty}\mathbb{E}(d_{2}(\widetilde{m}(X),m(X))| \mathcal{D}_{train})\to 0$, in probability.
 \item  The population quantile \(q_{1-\alpha} = \inf\{t \in \mathbb{R} : G(t) = \mathbb{P}(d_{2}(Y, m(X)) \leq t) = 1-\alpha\}\) is unique and is a continuity point of the function \(G(\cdot)\).
 \end{enumerate}
	\end{assumption}


	\begin{theorem}\label{th:conshom}
		Under Assumption \ref{1Smet}, the estimated prediction region $\widetilde{C}^{\alpha}(x;\mathcal{D}_{n})$ obtained using Algorithm \ref{alg:metd1} satisfies

		\begin{equation}
			\lim_{n\to \infty} \mathbb{E}(\varepsilon(\widetilde{C}^{\alpha},X)|\mathcal{D}_{n})\to 0  \text{ in probability. }
		\end{equation}
		
	\end{theorem}
	
	\begin{proof}
		See Appendix. 
	\end{proof}

\begin{assumption}\label{1S2met}
Suppose that for all \(t \in \mathbb{R}\), the distribution function of distances \(G(t, x) = \mathbb{P}(d_2(Y, m(x)) \leq t \mid X = x)\) is uniformly Lipschitz  with constant \(L\) in the sense that for all \(t, t' \in \mathbb{R}\) and for all \(x \in \mathcal{X}\), we have \(|G(t, x) - G(t', x)| \leq L |t - t'|\).
\end{assumption}

\begin{example}
\noindent	
Suppose that $\mathcal{X}= \mathbb{R}^{p}$, $\mathcal{Y}= \mathbb{R}^{m},$ $d_{1}(\cdot, \cdot)=d_{2}(\cdot,\cdot)=\norm{\cdot-\cdot}_{2}$ ($L^{2}$ norm), and consider the regression model
		
		\begin{equation}\label{eqn:gen}
		Y= m(X)+\epsilon, 
		\end{equation}
	\noindent	  where $\epsilon$ is a random vector taking values in $\mathbb{R}^{m}$, and independent from $X$. Let $r=\|\epsilon\|_{2}$, and suppose that $r$ admits a density $f_r$ satisfying
\[
\sup_{x\geq 0} f_r(x)\leq L<\infty.
\]
Then the distribution function $G$ is Lipschitz continuous.
  \end{example}
    
 \begin{remark}
Suppose that the assumptions of Example $3$ are satisfied. Consider any estimator \(\widetilde{m}\) of \(m\), and define the distribution of the distances:
\[
G^{*}(t, x) = \mathbb{P}(d_{2}(Y, \widetilde{m}(X)) \leq t \mid X = x).
\] 
We can express this probability in terms of the perturbation introduced by the estimator:
\[
G^{*}(t, x) = \mathbb{P}(\| \epsilon + (m(X) - \widetilde{m}(X)) \| \leq t \mid X = x).
\] 
Then for all $x$, the function \(G^{*}(t, x)\) is Lipschitz continuous.
\end{remark}

We now provide a non-asymptotic bound  on the expected coverage error of both the estimated prediction region and the oracle prediction region. Intuitively, this error can be controlled by the estimation errors of the regression function $m$ and of the calibration quantile $q_{1-\alpha}$.

	\begin{proposition}\label{the:minimax-1}
		Under Assumptions \ref{1Smet} and \ref{1S2met},  and assuming that $G^{*}(t,x)=\mathbb{P}(d_{2}(Y,\widetilde{m}(X))\leq t\mid X=x),$  is uniformly Lipschitz, we  have:
		
		\begin{equation*}
		\mathbb{E}(\varepsilon(\widetilde{C}^{\alpha},X)|\mathcal{D}_{n})\leq C(\mathbb{E}(d_{2}(\widetilde{m}(X),m(X))|\mathcal{D}_{train}) + (|\widetilde{q}_{1-\alpha}-{q}_{1-\alpha}|)),
		\end{equation*}
		
		\noindent where $C$ is a positive constant depending  on the Lipschitz constants of the functions $G$ and $G^{*}$.
  \end{proposition}
		
	\begin{proof}
		See Appendix. 
	\end{proof}

\begin{example}
\noindent
Suppose that $\mathcal{X} = \mathbb{R}^p$, $\mathcal{Y} = \mathbb{R}^m$, and $d_1(\cdot, \cdot) = d_2(\cdot, \cdot) = \|\cdot - \cdot\|_{2}$ (the $L^{2}$ norm). Consider the regression model
\begin{equation}\label{reg:def}
Y = m(X) + \epsilon,
\end{equation}
\noindent where $\epsilon$ is a random vector taking values in $\mathbb{R}^m$ and is independent of $X$. Let $f$ denote the density function of $\epsilon$. For this model, Proposition \ref{the:minimax-1} implies

	\begin{equation*}
		\mathbb{E}(\varepsilon(\widetilde{C}^{\alpha},X)|\mathcal{D}_{n})\leq C(\mathbb{E}(d_{2}(\widetilde{m}(X),m(X))|\mathcal{D}_{train}) + 4(|\widetilde{q}_{1-\alpha}-{q}_{1-\alpha}|)).
		\end{equation*}

\end{example}

\begin{remark}
In the absence of covariates and assuming the Fréchet mean \(m\) is known, the estimation rate depends solely on the convergence speed of the quantile estimator for pseudo-residuals 
\[
r = d\bigl(Y, m(X)\bigr).
\]
In this homoscedastic setting, our bounds extend the optimality results of \cite{gyofi2020nearest} for the univariate empirical distribution, yielding a rate of
$O\!\bigl(\sqrt{\ln n / n}\bigr).$ When \(m\) is unknown but can be estimated, at the parametric rate \(O(n^{-1/2})\), our algorithm can achieve fast convergence rates. For an explicit model-specific rate calculation, one must account for the error in estimating \(m\) and propagate it into the quantile estimation for \(q_{1-\alpha}\). To do this, one can invoke measurement-error quantile-estimation results from \cite{hansmann2017estimation}. In conclusion, the estimation of the prediction radius does not suffer from the dimensionality of the predictor because, under metric homoscedasticity, it reduces to estimating a single unconditional scalar quantile.
\end{remark}

  
  \begin{remark}
\noindent More specific non-asymptotic results can be derived for homoscedastic cases and univariate responses, as demonstrated in Corollary 1 of \cite{foygel2021limits}. In their work, the authors impose more stringent assumptions to bound the Lebesgue measure, including requirements such as the existence of a density function for the random error $\epsilon$ and symmetrical conditions, alongside stability criteria for the underlying mean regression functions $m$. In our setting, we construct the level set using balls without imposing any smoothness assumptions. We only assume the integrated asymptotic consistency of the regression function $m$.

	\end{remark}


	\subsection{Heteroscedastic case}\label{sec:hetero}


\begin{algorithm}[htbp]
\footnotesize
\caption{Uncertainty quantification algorithm for the heteroscedastic case with data-driven selection of \(k\)}
\label{algor:hete}

\begin{enumerate}
\setlength{\itemsep}{0.25em}

\item Split the data into \(\mathcal D_{\mathrm{train}}\) and
\(\mathcal D_{\mathrm{test}}\), with
\(n_{\mathrm{2}}=|\mathcal D_{\mathrm{test}}|\). Estimate
\(m(\cdot)\) by \(\widetilde m(\cdot)\) using
\(\mathcal D_{\mathrm{train}}\). For each
\((X_i,Y_i)\in\mathcal D_{\mathrm{test}}\), compute
\[
\widetilde r_i
=
d_2(Y_i,\widetilde m(X_i)).
\]

\item Let \(\mathcal K_n\) be a nonempty finite grid of candidate
neighborhood sizes. For \(x\in\mathcal X\) and \(k\in\mathcal K_n\),
let
\[
 N_k(x)
\equiv
 N_k(x;\mathcal D_{\mathrm{test}})
\]
denote the indexes of the $k$NNs of \(x\) among the
predictor values in \(\mathcal D_{\mathrm{test}}\). Define
\[
\widetilde q_{1-\alpha}(x;k)
=
\inf\left\{
t\geq 0:
\frac{1}{k}
\sum_{j\in\mathcal N_k(x)}
\mathbb I\left\{\widetilde r_j\leq t\right\}
\geq 1-\alpha
\right\}.
\]

\item For each \(k\in\mathcal K_n\), define
\[
\widetilde C^{\alpha,k}(x)
=
\mathcal B\left\{
\widetilde m(x),
\widetilde q_{1-\alpha}(x;k)
\right\}.
\]

\item Define the global empirical coverage and its deviation from the
nominal level by
\[
\widehat{\mathrm{Cov}}_{\mathrm{glob}}(k)
=
\frac{1}{n_{\mathrm{2}}}
\sum_{(X_i,Y_i)\in\mathcal D_{\mathrm{test}}}
\mathbb I(
Y_i\in\widetilde C^{\alpha,k}(X_i)),
\]
and
\[
\Delta_{\mathrm{glob}}(k)
=
\left|
\widehat{\mathrm{Cov}}_{\mathrm{glob}}(k)
-
(1-\alpha)
\right|.
\]

\item Let
\[
\mathcal X_{\mathrm{test}}
=
\left\{
X_i:(X_i,Y_i)\in\mathcal D_{\mathrm{test}}
\right\}.
\]
For each \(x\in\mathcal X_{\mathrm{test}}\), define the local empirical
coverage deviation
\[
\Delta_{\mathrm{loc}}(x;k)
=
\left|
\frac{1}{k}
\sum_{i\in\mathcal N_k(x)}
\mathbb I(Y_i\in\widetilde C^{\alpha,k}(X_i))
-
(1-\alpha)
\right|,
\]
and let
\[
\Delta_{\mathrm{loc}}^{\max}(k)
=
\max_{x\in\mathcal X_{\mathrm{test}}}
\Delta_{\mathrm{loc}}(x;k).
\]

\item Given \(\varepsilon>0\), define the set of admissible
neighborhood sizes
\[
\mathcal K_{\varepsilon}
=
\left\{
k\in\mathcal K_n:
\Delta_{\mathrm{loc}}^{\max}(k)\leq\varepsilon,
\quad
\Delta_{\mathrm{glob}}(k)\leq\varepsilon
\right\}.
\]
If \(\mathcal K_{\varepsilon}\neq\varnothing\), select
\[
k_{\varepsilon}^{\ast}
=
\min\mathcal K_{\varepsilon}.
\]
Otherwise, select
\[
k_{\varepsilon}^{\ast}
\in
\arg\min_{k\in\mathcal K_n}
\max\left\{
\Delta_{\mathrm{loc}}^{\max}(k),
\Delta_{\mathrm{glob}}(k)
\right\},
\]
choosing the smallest minimizer in the event of a tie.

\item For any new point \(x\in\mathcal X\), return
\[
\widetilde C^{\alpha,\varepsilon}(x)
=
\mathcal B(
\widetilde m(x),
\widetilde q_{1-\alpha}
\left(x;k_{\varepsilon}^{\ast}\right)
).
\]

\end{enumerate}
\end{algorithm}

\noindent	The homoscedastic assumption can be too restrictive in metric spaces that lack a linear structure, making a local approach more appropriate. In this paper, we adopt the $k$NN regression algorithm \cite{fix1951discriminatory, biau2015lectures} for this purpose, as it is both computationally efficient and easy to implement.
	
\noindent	In heteroscedastic methods, given a point \(x\in \mathcal{X}\) with distance metric \(\rho_{\mathcal{X}}\), we repeat the same steps as in the homoscedastic algorithm, but the radius estimation only considers data points in a neighborhood of \(x\).  
Let \(S_2\) be an index set and \(\mathcal{D}_{\mathrm{test}}\) a set of predictor values (with \(n_2 = |\mathcal{D}_{\mathrm{test}}|\)).  
For each \(x\in\mathcal{X}\), define  
\[
N_k(x) = \left\{ i\in S_2: \rho_{\mathcal{X}}(X_i,x)\leq \rho_{\mathcal{X}_{n_2,k}}(x) \right\},
\]  
where \(\rho_{\mathcal{X}, n_2, k}(x)\) denotes the distance from \(x\) to its \(k\)-th-nearest-neighbor among the predictor values in \(\mathcal{D}_{\mathrm{test}}\), measured using the metric \(\rho_{\mathcal{X}}\).

\noindent Algorithm~\ref{alg:hete} provides a detailed description of the proposed procedure, including a data-driven criterion for selecting the number of neighbors \(k\). The method does not require smoothness assumptions or repeated smoothing steps. Instead, it balances marginal and local empirical calibration while favoring prediction regions that adapt to the local geometry of the predictor space. Its main computational operation is a nearest-neighbor search, for which efficient approximate algorithms are available in large-scale settings.

\noindent As in split conformal prediction \cite{Sample-Conditional_Coverage}, the center is estimated on a training sample, while the calibration sample is used to compute residual scores and estimate the prediction radius through empirical order statistics. Standard split conformal prediction uses a single global quantile to target marginal coverage. In contrast, our method computes local quantiles within $k$NN neighborhoods, allowing the radius to adapt to heteroscedasticity. This strategy is motivated by \cite{Sample-Conditional_Coverage}, who show that covariate-adaptive quantile thresholds estimated on held-out data can achieve strong approximate conditional-coverage properties. Their results support the general calibration principle, although they do not directly imply minimax optimality for our specific $k$NN construction.

\noindent The proposed tuning strategy is motivated by relaxed conditional calibration. Following the perspective of \citet{Conformal_prediction_with}, conditional validity may be relaxed by requiring calibration over a restricted function class \(\mathcal{F}\), rather than over all measurable subsets of the covariate space. Different choices of \(\mathcal{F}\), ranging from constant functions to finite-dimensional classes of subgroup indicators, provide different compromises between marginal and conditional validity. In our setting, fixed linear or subgroup-based classes are replaced by a data-adaptive family of metric neighborhoods. For a fixed \(k\), define the empirical coverage residual  
\[
E_i^{(k)} = \mathbb I\bigl\{ Y_i\in\widetilde C^{\alpha,k}(X_i) \bigr\} - (1-\alpha),
\]  
\noindent and the \(k\)NN weight  
\[
w_{x,k}(X_i) = \frac{1}{k}\,\mathbb I\bigl\{ i\in\mathcal N_k(x) \bigr\}.
\]  
\noindent Let \(\mathcal X_{\mathrm{test}} = \{ X_i : (X_i,Y_i)\in\mathcal D_{\mathrm{test}} \}\).  
The worst local empirical coverage deviation is then  
\[
\Delta_{\mathrm{loc}}^{\max}(k) = \max_{x\in\mathcal X_{\mathrm{test}}}
\Bigl| \sum_{i=1}^{n_{\mathrm{test}}} w_{x,k}(X_i) \, E_i^{(k)} \Bigr|,
\]  
\noindent and the global deviation is  
\[
\Delta_{\mathrm{glob}}(k) = \Bigl| \frac{1}{n_{\mathrm{test}}}
\sum_{i=1}^{n_{\mathrm{test}}} \mathbb I\bigl\{ Y_i\in\widetilde C^{\alpha,k}(X_i) \bigr\}
- (1-\alpha) \Bigr|.
\]  
\noindent The set of admissible neighborhood sizes is  
\[
\mathcal K_{\varepsilon} = \bigl\{ k\in\mathcal K_{\mathrm{grid}} :
\Delta_{\mathrm{loc}}^{\max}(k)\leq\varepsilon,\;
\Delta_{\mathrm{glob}}(k)\leq\varepsilon \bigr\},
\]  
\noindent and we select \(k_{\varepsilon}^{\ast} = \min \mathcal K_{\varepsilon}\). Thus, the method selects the smallest neighborhood size satisfying the prescribed global and local empirical calibration constraints. The parameter $k$ controls the trade-off between local adaptation and statistical stability: smaller values yield more localized estimates of the conditional prediction radius, whereas larger neighborhoods provide more stable empirical quantiles. This trade-off is particularly relevant because the prediction regions are restricted to isotropic metric balls, which may be overly conservative in heterogeneous regions of the predictor space. At the same time, excessively small neighborhoods may produce unstable quantile estimates and overfit the calibration sample. The global and local calibration constraints mitigate this instability, while selecting the smallest admissible $k$ favors the greatest degree of local adaptation compatible with acceptable empirical calibration. More generally, the proposed selection framework is flexible and can accommodate alternative objectives. For example, leave-one-out estimates could be used to select, among the admissible values of $k$, the value minimizing the average prediction radius or another measure of predictive efficiency.

\noindent The criterion can therefore be interpreted as imposing approximate conditional calibration over a data-adaptive class of \(k\)NN metric neighborhoods. It does not require exact conditional coverage at every individual covariate value, which is generally impossible without strong distributional assumptions. Instead, it targets calibration over neighborhoods whose resolution is controlled by \(k\): smaller values correspond to more local notions of coverage, whereas larger values produce more stable calibration that is closer to marginal conformal validity.

\noindent Compared with methods based on a fixed linear class \(\mathcal{F}\), the proposed construction does not require the user to specify features or subgroup indicators. It only requires a metric on the predictor space and can therefore be applied directly to non-Euclidean predictors. For more complex predictor spaces, the neighborhoods may also be constructed using anchor-based metric representations, in which observations are represented through their distances to a finite collection of representative points (see for example \cite{matabuena2026randomeffectsalgorithmrandomobjects}). Provided that the selected neighborhood size remains within the asymptotic regime required by the theory, the resulting procedure preserves consistency. In the numerical experiments, we set \(\varepsilon=0.05\).

\begin{remark}    
Although Algorithm \ref{alg:hete} can be consistent,  finite-sample validity for this adaptive heteroscedastic construction is not established here. To address this limitation and propose a heteroscedastic algorithm with non-asymptotic performance guarantees, we suggest extending the conformal quantile methodology introduced by \cite{romano2019conformalized} to our context. Specifically, we aim to achieve $\mathbb{P}(Y \in \widetilde{C}^{\alpha}(X;\mathcal{D}_{n})) \geq 1 - \alpha$. The core steps are outlined below:

\begin{enumerate}
    \item \textbf{Estimation of $m(\cdot)$:} Utilize the dataset $\mathcal{D}_{train}$ to estimate the regression function $m(\cdot)$.
    
    \item \textbf{Estimation of $r(\cdot)$:} Employ the random sample $\mathcal{D}_{test}$ and the pairs $(X_{i}, d_{2}(Y_{i}, \widetilde{m}(X_{i})))$ for each $i \in [S_2]$ to estimate the radius function $r(\cdot)$.
    
    \item \textbf{Calibration of $\widetilde{r}(x)$:} In an additional data set $\mathcal{D}_{test2}$, we calibrate the radius $\widetilde{r}(x)$ to obtain non-asymptotic guarantees using the scores $S_{i} = d_{2}(Y_{i}, \widetilde{m}(X_{i})) - \widetilde{r}(X_{i})$. The empirical quantile at level $1-\alpha$, denoted by $\widetilde{w}_{\alpha}$, is then used to define the final radius function as $\widehat{r}(x) = \widetilde{r}(x) + \widetilde{w}_{\alpha}$.
\end{enumerate}
	\end{remark}

	\subsubsection{Statistical Theory}

\noindent	The following technical assumptions are introduced to guarantee that the proposed heteroscedastic $k$NN-algorithm is asymptotically consistent.

\begin{assumption}\label{alg:hete}
    \begin{enumerate}
        \item $n_1,n_2 \to \infty$.
        \item $\widetilde{m}$ is a consistent estimator in the sense that $\mathbb{E}(d_{2}(\widetilde{m}(X), m(X))|\mathcal{D}_{\text{train}}) \to 0$  in probability as $n_1 \to \infty$.
        \item As $n_{2}\to \infty$, $k \to \infty$, and $\frac{k}{n_{2}} \to 0$.
       \item Except in a set of probability zero, for all \(x \in \mathcal{X}\), the pointwise population quantile \(q_{1-\alpha}(x) = \inf\{t \in \mathbb{R} : G(t , x) = \mathbb{P}(d_{2}(Y, m(x)) \leq t \mid X=x) = 1-\alpha\}\) exists uniquely and  is a continuity point of the function \(G(\cdot, x)\).
\item Except in a set of probability zero, for all \(x \in \mathcal{X}\), \(\mathbb{E}(d_{2}(Y, m(x))\mid X=x) < \infty\).
    \end{enumerate}
\end{assumption}



 The following result applies to a deterministic sequence
$k=k_n$ satisfying
\[
k_n\to\infty,
\qquad
k_n/n_2\to0.
\]

	\begin{theorem}\label{theorem:hetero}
		Under Assumption \ref{alg:hete}, the uncertainty quantification estimator \(\widetilde{C}^{\alpha}(\cdot)\) obtained using Algorithm \ref{algor:hete} for a deterministic $k$ satisfies:
		\begin{equation*}
		\lim_{n\to \infty} \mathbb{E}(\varepsilon(\widetilde{C}^{\alpha},X)|\mathcal{D}_{n}) \to 0. 
		\end{equation*}
	\end{theorem}

\begin{proof}
The proof is given in the Appendix.
\end{proof}

\noindent Given a grid of candidate values \(\mathcal K_n\), chosen to satisfy the conditions required for consistency, we select an optimized value of \(k\) according to Algorithm~\ref{algor:hete}, with the aim of improving conditional calibration. The following results show that, for an appropriate choice of \(\mathcal K_n\), the proposed selection rule is universally consistent.

\begin{corollary}
Let \(\mathcal K_n\) be a nonempty finite candidate grid, possibly data-dependent candidate grid such that
\[
\min_{k\in\mathcal K_n} k \longrightarrow \infty,
\qquad
\frac{\max_{k\in\mathcal K_n} k}{n_{2}}
\longrightarrow 0
\]
in probability. Let \(\widetilde{k}_n\) be any measurable selection rule satisfying
\(\widetilde{k}_n\in\mathcal K_n\) almost surely. Then, under the assumptions of
Theorem~\ref{theorem:hetero},
\begin{equation*}
  \lim_{n\to \infty}  \mathbb{E}(
\varepsilon (
\widetilde C^{\alpha,\widetilde k_n},X)
\mid\mathcal D_n)
\longrightarrow 0
\end{equation*}

\end{corollary}

\noindent The following additional technical assumptions are used to derive a non-asymptotic
upper bound for the local coverage error of the proposed heteroscedastic
\(k\)NN uncertainty quantification procedure.

\begin{assumption}\label{ass:meas_quant}
Fix \(\alpha\in(0,1)\). Assume that:
\begin{enumerate}

    \item Suppose that \(m\), \(G\), and their corresponding estimators \(\widetilde m\) and \(\widetilde G\) are measurable functions.

    \item Except on a set of probability zero, for all \(x\in\mathcal X\),
    the conditional quantile
    \[
    q_{1-\alpha}(x)
    =
    \inf\{t\geq 0:G(t,x)\geq 1-\alpha\}
    \]
\noindent    is finite, uniquely defined, and is a continuity point of
    \(G(\cdot,x)\).

    \item There exists \(L<\infty\) such that, except on a set of probability
    zero, for all \(s,t\geq 0\) and almost every \(x\in\mathcal X\),
    \[
    |G(t,x)-G(s,x)|
    \leq
    L|t-s|.
    \]

\item The following integrability conditions hold:
\[
\mathbb E
\left[
d_2(\widetilde m(X),m(X))
\,\middle|\,
\mathcal D_{\mathrm{train}}
\right]
<\infty,
\]
and
\[
\mathbb E
\left[
\left|
\widetilde q_{1-\alpha}(X)-q_{1-\alpha}(X)
\right|
\, \mid
\mathcal D_n
\right]
<\infty.
\]
\end{enumerate}
\end{assumption}

\begin{proposition}\label{the:minimax1}
Suppose that Algorithm \ref{algor:hete} is well defined and that Assumption
\ref{ass:meas_quant} holds. Then the uncertainty quantification estimator
\(\widetilde C^{\alpha}(\cdot)\) satisfies
\[
\mathbb{E}
\left[
\varepsilon(\widetilde C^{\alpha},X)
\,\middle|\,
\mathcal D_n
\right]
\leq
C(\mathbb{E}
\left[
d_2(\widetilde m(X),m(X))
\,\middle|\,
\mathcal D_{\mathrm{train}}
\right]
+
4
\mathbb{E}(\left|
\widetilde q_{1-\alpha}(X)-q_{1-\alpha}(X)\mid
\mathcal D_n)
\right ),
\]
where \(C>0\) is a constant depending only on the Lipschitz constant \(L\).
\end{proposition}

\begin{proof}
The proof is given in the Appendix.
\end{proof}

\noindent Suppose that \(\mathcal X=\mathbb R^p\) and
\(\mathcal Y=\mathbb R^m\), with $d_1(\cdot,\cdot)=d_2(\cdot,\cdot)=\|\cdot-\cdot\|_{2}.$ Consider the heteroscedastic regression model
\[
Y=m(X)+\sigma(X)\odot \epsilon,
\]
\noindent where \(m:\mathbb R^p\to\mathbb R^m\), the symbol \(\odot\) denotes componentwise multiplication, and
\(\sigma:\mathbb R^p\to\mathbb R^m\), \(\epsilon\) is
independent of \(X\), and $\mathbb E(\epsilon)=0$,
$\operatorname{Cov}(\epsilon)=I_m.$ Assume that the components of \(\sigma\) are bounded away from zero, and the conditional
distribution of
\[
d_2(Y,m(X))
=
\|\sigma(X)\odot\epsilon\|_2
\]
\noindent has a distribution function \(G(\cdot,x)\) that is Lipschitz in its first
argument, uniformly over \(x\) (i.e., \(|G(t_1,x)-G(t_2,x)|\le L|t_1-t_2|\) for some \(L>0\) and $\forall x\in \mathcal{X})$. If Assumption \ref{ass:meas_quant} holds, we have the rate decomposition for
Proposition \ref{the:minimax1}.

\begin{remark}
Proposition \ref{the:minimax1} shows that the local coverage error of the
heteroscedastic procedure is controlled by two quantities. The first term is
the estimation error of the regression function \(m\). The second term is the
error incurred by estimating the local conditional quantile
\(q_{1-\alpha}(x)\). This second term is specific to the heteroscedastic
setting, where the prediction radius depends on \(x\). For the \(k\)NN radius estimator used in Algorithm \ref{algor:hete}, the
second term corresponds to the local quantile estimation error based on the nearest-neighbor residuals. In related work, \cite{gyofi2020nearest} study nearest-neighbor prediction intervals and obtain rates of order
\[
O\left(
\frac{\log n}{n^{1/(p+2)}}
\right),
\qquad p\geq 1,
\]
when the conditional regression function \(m\) is known. This rate depends on
the dimension \(p\) of the covariate space, in contrast with the homoscedastic
case, where the prediction radius is global rather than local.
\end{remark}

\subsection{Regression algorithms in metric spaces: Random Forest, Global Fréchet regression method, and \(k\)NN}\label{sec:global}
In this section, we introduce a general background on regression modeling in metric spaces needed to implement our predictive algorithms in practice. We first introduce two nonparametric regression alternatives: $k$NN and Random Forests. Finally, we introduce the global Fréchet regression algorithm as a special type of parametric  model for metric-space-valued outcomes. Although more recent approaches, including deep-learning methods for metric-space data \cite{Iao03072025}, have been proposed, we restrict our attention to these methods for the sake of brevity.



\noindent Before proceeding, recall that the target $m$ is the conditional Fréchet mean

\[
m(x)=\arg\min_{y\in\mathcal{Y}}\,
      \mathbb{E}\bigl[d_{1}^{2}(Y,y)\mid X=x\bigr],
\]

\noindent whose estimation must be made feasible in practice.

\medskip
\noindent\textbf{\(k\)NN estimator.}
Given an i.i.d. sample \(\mathcal{D}_{n}=\{(X_i,Y_i)\}_{i=1}^{n} \subset \mathcal{X}\times \mathcal{Y}\), the simple yet effective estimator calculates the \emph{Fréchet mean} locally

\[
\widetilde{m}_{k}(x)
    =\arg\min_{y\in\mathcal{Y}}
      \frac{1}{k}\sum_{i\in N_{k}(x)}d_{1}^{2}(Y_i,y),
\]

\noindent where \(N_k(x)\) denotes the indexes of the $k$NN of the point \(x\) in the training sample \(\mathcal{D}_{\text{train}}\), with respect to a distance \(d_{1}\) that measures the similarity of the elements in the predictor space. A generalization of this algorithm to non-Euclidean predictor spaces, based on the concept of \(k\)-prototypes, was proposed by \citet{Cohen2022}. The authors established the universal consistency of the estimator under mild regularity conditions--conditions that are also satisfied by our uncertainty quantification algorithms.

\medskip
\noindent\textbf{Random Forest regressogram estimator.}
Recently, a Random Forest for non--Euclidean predictors and responses was introduced by \citet{capitaine2024frechet}.  Let \(\pi_n\) denote a partitioning rule that can be estimated for the specific split tree rules introduced in \citep{capitaine2024frechet}.  The corresponding Fréchet regressogram estimator is

\[
\widetilde{m}(x)
  =\arg\min_{y\in\mathcal{Y}}
     \frac{1}{n}\sum_{i=1}^{n}
       d_{1}^{2}(Y_i,y)\,
       \mathbb{I}\{X_i\in\pi_n[x]\},
\]
\noindent where \(\pi_n[x]\) is the unique cell containing \(x\). Over this tree--algorithm, the authors define a random forest algorithm.

Under the regularity conditions specified by \citeauthor{capitaine2024frechet} for universal consistency, our prediction region algorithm is consistent. However, these theoretical guarantees require that the metric space have a finite diameter and that the associated regression functions $m$ satisfy smoothing (differentiability) assumptions that are not imposed by the \(k\)NN based approach of \cite{cohen2022learning}.
	
\hspace{0.2cm}
\noindent \textbf{Global Fréchet Algorithm} Now, we focus on a specific parametric and less flexible regression model designed for outcomes valued in the metric space, known as \emph{global Fréchet regression model} \citep{petersen2019frechet}.

The global Fréchet regression model provides a natural extension of the standard linear regression model from Euclidean spaces to metric spaces. 

\noindent 
In particular, \cite{petersen2019frechet}
propose to model the regression function in the form
\begin{equation*}
m(x)=\underset{y \in \mathcal{Y}}{\operatorname{argmin}} \mathbb{E}\left[\omega(x, X) d^{2}(Y,y)\right]~,
\end{equation*}
where the weight function $\omega$ is defined by
\begin{equation*}
\omega(x,z)=1+(z-\mu)^{\top} \Sigma^{-1}(x-\mu) 
\end{equation*}
with $\mu=\mathbb{E}(X)$ and $\Sigma=\operatorname{Cov}(X)$.
As \cite{petersen2019frechet} show, in the case of $\mathcal{Y}=\mathcal{H}$, where $\mathcal{H}$ is a separable Hilbert space with an inner product $\langle \cdot, \cdot \rangle,$ 
this model reduces to standard linear regression.





\noindent While in our applications we primarily focus on developing estimators in the global Fréchet regression model, it is important to emphasize that our uncertainty quantification algorithm can be applied to any regression technique, including random forests in metric spaces. It should be noted, however, that if the global Fréchet models do not align with the conditional Fréchet mean \( m \), the resulting estimators will be inconsistent. 


\noindent To estimate  the conditional mean function $m(x)$ under the
 global Fréchet model from a random sample $\mathcal{D}_n=\{(X_i,Y_i)\}_{i=1}^n$, we may solve the counterpart empirical problem

\begin{equation}\label{eqn:linear}
	\widetilde{m}(x)= \argminB_{y \in \mathcal{Y}} \frac{1}{n} \sum_{i=1}^{n} [\omega_{in}(x) d_{1}^{2}(y,Y_i)],
\end{equation}

\noindent where $\omega_{in}(x)= \left[ 1+(X_{i}-\overline{X})^{\top}\widetilde{\Sigma}^{-1}(x-\overline{X})\right],$ with $\overline{X}= \frac{1}{n} \sum_{i=1}^{n} X_i$, and $\widetilde{\Sigma}= \frac{1}{n-1} \sum_{i=1}^{n} (X_i-\overline{X})(X_i-\overline{X})^{\top}.$

\subsection{Sequential extension under temporal dependence}
\label{sec:time_series_summary}

The proposed center--radius framework can also be extended to metric-space-valued time series. Let \((Y_t)_{t\in\mathbb Z}\) be a strictly stationary and ergodic process taking values in a metric space \((\mathcal Y,d)\), and let \(\mathcal F_t=\sigma(Y_s:s\leq t)\). At each time \(t\), the goal is to construct a sequential prediction set for \(Y_{t+1}\) using only the information available in \(\mathcal F_t\). Ideally, one would use the conditional Fréchet mean
\[
m_t\in\arg\min_{y\in\mathcal Y}
\mathbb E[d^2(Y_{t+1},y)\mid\mathcal F_t],
\]
together with the conditional \((1-\alpha)\)-quantile \(q_t\) of the residual \(d^{2}(Y_{t+1},m_t)\). This yields the oracle prediction ball
\[
\mathcal C_t^\alpha
=
\{y\in\mathcal Y:d^{2}(y,m_t)\leq q_t\},
\]
which attains conditional coverage \(1-\alpha\) under the usual quantile regularity condition, for example when \(F_t(q_t)=1-\alpha\).

\noindent In practice, \(m_t\) is replaced by an \(\mathcal F_t\)-measurable estimator \(\widetilde m_t\), and the radius is estimated from past residuals using a nearest-neighbor expert-aggregation procedure \citep{cesa2006prediction,biau2011sequential}. For different memory lengths \(k\) and neighborhood sizes, the method compares the current lag block
\[
Z_t^{(k)}=(Y_{t-k+1},\ldots,Y_t)
\]
with similar historical blocks and estimates the corresponding conditional residual quantile. These local quantile experts are then combined through exponential weights based on the pinball loss, producing a data-adaptive radius \(\widetilde q_t\). The resulting prediction set is
\[
\widetilde{\mathcal C}_t^\alpha
=
\{y\in\mathcal Y:d^{2}(y,\widetilde m_t)\leq \widetilde q_t\}.
\]

\noindent This construction does not require independence, exchangeability, a parametric time-series model, explicit mixing-rate assumptions, or smoothing assumptions on the temporal dynamics. Its guarantees instead rely on stationarity, ergodicity, and regularity conditions on the metric space and the conditional residual distribution. Following the sequential prediction framework of \cite{biau2011sequential}, and under the assumptions stated in Supplemental Material~C, the proposed procedure satisfies the following asymptotic time-average calibration property:

\[
\frac{1}{n}\sum_{t=1}^{n}
\mathbb P\left(Y_{t+1}\in\widetilde{\mathcal C}_t^\alpha\right)
\longrightarrow 1-\alpha.
\]

\noindent Thus, the result should be interpreted as an asymptotic time-average calibration guarantee, rather than as finite-sample conformal validity. Stronger forms of conditional coverage may be established under additional assumptions on the stochastic process and the local quantile estimators, but such extensions are beyond the scope of this manuscript. The full algorithm, technical assumptions, proof, fixed-horizon multistep extension, and a practical real-data case study are provided in Supplemental Material~C.



\section{Simulation Study}

\noindent
We conducted a simulation study with independent data to evaluate the empirical performance of the proposed uncertainty quantification framework. In the main manuscript, we focus on scalar-on-scalar regression, which allows comparison with state-of-the-art methods. Additional simulations with metric-space-valued outcomes are reported in the Supplemental Material. The simulation study reported here has two main objectives. First, we compare the conditional coverage of the proposed procedures with that of state-of-the-art methods in scalar and metric-space settings, including the approach of \citep{zhou2024conformalinferencerandomobjects}. Second, we assess the efficiency gains achieved by the heteroscedastic procedure with data-adaptive selection of $k$, relative to the homoscedastic method, under both homoscedastic and heteroscedastic data-generating mechanisms. For the scalar-on-scalar experiments, we consider several signal-to-noise regimes and evaluate performance using global measures, including mean squared error, ($L^2$) error, and worst-slice coverage. For brevity, several complementary results and examples involving non-Euclidean data structures are relegated to the Supplemental Material.

\begin{table}[ht!]
\centering
\scriptsize
\setlength{\tabcolsep}{3pt}
\renewcommand{\arraystretch}{1.15}
\resizebox{\textwidth}{!}{%
\begin{tabular}{l|rr|rr||rr|rr||rr|rr}
\toprule
& \multicolumn{4}{c||}{\textbf{$1-\alpha=0.50$}} & \multicolumn{4}{c||}{\textbf{$1-\alpha=0.80$}} & \multicolumn{4}{c}{\textbf{$1-\alpha=0.95$}} \\
\cmidrule(lr){2-5}\cmidrule(lr){6-9}\cmidrule(lr){10-13}
& \multicolumn{2}{c|}{$n=1000$} & \multicolumn{2}{c||}{$n=3000$} & \multicolumn{2}{c|}{$n=1000$} & \multicolumn{2}{c||}{$n=3000$} & \multicolumn{2}{c|}{$n=1000$} & \multicolumn{2}{c}{$n=3000$} \\
\cmidrule(lr){2-3}\cmidrule(lr){4-5}\cmidrule(lr){6-7}\cmidrule(lr){8-9}\cmidrule(lr){10-11}\cmidrule(lr){12-13}
\textbf{Method} & \textit{ISCE} & \textit{WSC} & \textit{ISCE} & \textit{WSC} & \textit{ISCE} & \textit{WSC} & \textit{ISCE} & \textit{WSC} & \textit{ISCE} & \textit{WSC} & \textit{ISCE} & \textit{WSC} \\
\midrule
\multicolumn{13}{l}{\textsc{Setting 1}} \\[1pt]
\textbf{Heteroscedastic} & \textbf{0.083} & \textbf{0.370} & \textbf{0.054} & \textbf{0.389} & \textbf{0.068} & \textbf{0.678} & \textbf{0.044} & \textbf{0.703} & \textbf{0.040} & \textbf{0.864} & \textbf{0.024} & \textbf{0.891} \\
\textbf{Homoscedastic} & \textbf{0.299} & \textbf{0.176} & \textbf{0.297} & \textbf{0.174} & \textbf{0.212} & \textbf{0.437} & \textbf{0.213} & \textbf{0.441} & \textbf{0.094} & \textbf{0.717} & \textbf{0.093} & \textbf{0.723} \\
CQR & 0.060 & 0.372 & 0.041 & 0.399 & 0.051 & 0.696 & 0.031 & 0.719 & 0.031 & 0.883 & 0.019 & 0.902 \\
Profile & 0.136 & 0.215 & 0.134 & 0.222 & 0.138 & 0.529 & 0.138 & 0.537 & 0.094 & 0.723 & 0.095 & 0.729 \\
Dist & 0.172 & 0.273 & 0.163 & 0.320 & 0.144 & 0.577 & 0.137 & 0.586 & 0.059 & 0.887 & 0.050 & 0.951 \\
HPD & 0.232 & 0.195 & 0.231 & 0.224 & 0.210 & 0.425 & 0.197 & 0.466 & 0.222 & 0.515 & 0.186 & 0.582 \\
\midrule
\multicolumn{13}{l}{\textsc{Setting 2}} \\[1pt]
\textbf{Heteroscedastic} & \textbf{0.102} & \textbf{0.359} & \textbf{0.071} & \textbf{0.382} & \textbf{0.067} & \textbf{0.671} & \textbf{0.045} & \textbf{0.705} & \textbf{0.038} & \textbf{0.857} & \textbf{0.024} & \textbf{0.889} \\
\textbf{Homoscedastic} & \textbf{0.293} & \textbf{0.175} & \textbf{0.296} & \textbf{0.173} & \textbf{0.204} & \textbf{0.446} & \textbf{0.207} & \textbf{0.453} & \textbf{0.087} & \textbf{0.736} & \textbf{0.089} & \textbf{0.732} \\
CQR & 0.449 & 0.000 & 0.450 & 0.000 & 0.327 & 0.379 & 0.326 & 0.407 & 0.171 & 0.779 & 0.171 & 0.779 \\
Profile & 0.348 & 0.022 & 0.354 & 0.004 & 0.170 & 0.513 & 0.205 & 0.431 & 0.100 & 0.747 & 0.099 & 0.796 \\
Dist & 0.217 & 0.054 & 0.213 & 0.058 & 0.229 & 0.350 & 0.226 & 0.349 & 0.136 & 0.702 & 0.118 & 0.743 \\
HPD & 0.233 & 0.187 & 0.235 & 0.199 & 0.202 & 0.471 & 0.187 & 0.491 & 0.148 & 0.743 & 0.110 & 0.781 \\
\midrule
\multicolumn{13}{l}{\textsc{Setting 3}} \\[1pt]
\textbf{Heteroscedastic} & \textbf{0.095} & \textbf{0.371} & \textbf{0.074} & \textbf{0.384} & \textbf{0.062} & \textbf{0.680} & \textbf{0.045} & \textbf{0.707} & \textbf{0.029} & \textbf{0.876} & \textbf{0.019} & \textbf{0.896} \\
\textbf{Homoscedastic} & \textbf{0.270} & \textbf{0.187} & \textbf{0.273} & \textbf{0.183} & \textbf{0.178} & \textbf{0.499} & \textbf{0.183} & \textbf{0.486} & \textbf{0.067} & \textbf{0.797} & \textbf{0.068} & \textbf{0.795} \\
CQR & 0.271 & 0.038 & 0.270 & 0.024 & 0.239 & 0.442 & 0.241 & 0.445 & 0.125 & 0.812 & 0.121 & 0.834 \\
Profile & 0.221 & 0.240 & 0.219 & 0.251 & 0.118 & 0.590 & 0.130 & 0.573 & 0.036 & 0.868 & 0.042 & 0.853 \\
Dist & 0.227 & 0.245 & 0.226 & 0.257 & 0.137 & 0.579 & 0.135 & 0.589 & 0.049 & 0.860 & 0.043 & 0.859 \\
HPD & 0.256 & 0.203 & 0.255 & 0.207 & 0.185 & 0.517 & 0.173 & 0.521 & 0.101 & 0.728 & 0.100 & 0.724 \\
\midrule
\multicolumn{13}{l}{\textsc{Setting 4}} \\[1pt]
\textbf{Heteroscedastic} & \textbf{0.040} & \textbf{0.400} & \textbf{0.026} & \textbf{0.413} & \textbf{0.031} & \textbf{0.717} & \textbf{0.021} & \textbf{0.723} & \textbf{0.024} & \textbf{0.884} & \textbf{0.014} & \textbf{0.900} \\
\textbf{Homoscedastic} & \textbf{0.027} & \textbf{0.416} & \textbf{0.018} & \textbf{0.421} & \textbf{0.023} & \textbf{0.734} & \textbf{0.015} & \textbf{0.729} & \textbf{0.021} & \textbf{0.889} & \textbf{0.012} & \textbf{0.907} \\
CQR & 0.034 & 0.402 & 0.022 & 0.417 & 0.029 & 0.725 & 0.020 & 0.726 & 0.018 & 0.904 & 0.010 & 0.912 \\
Profile & 0.053 & 0.378 & 0.046 & 0.386 & 0.110 & 0.605 & 0.111 & 0.607 & 0.083 & 0.769 & 0.084 & 0.775 \\
Dist & 0.049 & 0.386 & 0.029 & 0.416 & 0.037 & 0.693 & 0.023 & 0.729 & 0.029 & 0.877 & 0.014 & 0.901 \\
HPD & 0.046 & 0.386 & 0.028 & 0.417 & 0.049 & 0.673 & 0.024 & 0.721 & 0.064 & 0.813 & 0.021 & 0.881 \\
\bottomrule
\end{tabular}
}
\caption{ISCE and local WSC for the conditional coverage for the four generative examples considered}
\label{tab:results1}
\end{table}

\subsection{Global Conditional Coverage Examination}

Most state-of-the-art conformal methods are designed primarily for scalar
responses. Moreover, reliable estimation of conditional coverage is generally
feasible only in low-dimensional settings and is most easily interpreted when
both the predictors and the responses are scalar. For these reasons, our initial
comparison of conditional coverage performance focuses on scalar-to-scalar
regression, although we also include the existing metric-space approach of
\citet{zhou2024conformalinferencerandomobjects}. In contrast to the simulation
settings considered in \citet{zhou2024conformalinferencerandomobjects}, where
the generative models are nearly deterministic, our settings involve
substantially higher noise levels. Here, we focus on the conditional coverage as a key parameter of interest:
$\mathbb{P}(Y \in \widetilde{C}^{\alpha}(X) \mid X)$,
 achieved by the optimized heteroscedastic $k$NN predictive
procedure; see Section~\ref{algor:hete} for details. We also test the
corresponding homoscedastic conformal version introduced in the Section \ref{sec:ho}.

\noindent We compare our procedure with established conformal methods for Euclidean
responses, including the HPD-split scores of \citet{izbicki2022cd}, denoted by
HPD and Dist in the result tables, respectively, and several variants of Conformalized
Quantile Regression (CQR) \citep{romano2019conformalized}. The HPD-split
implementation uses the \texttt{dist} and \texttt{hpd} functions from
\url{https://github.com/rizbicki/predictionBands}, while CQR is evaluated under
the default settings of the \texttt{cfsurvival} package. For completeness, we
also benchmark against the method of
\citet{zhou2024conformalinferencerandomobjects}, which is the only competitor
specifically designed for metric-space responses. For each simulation setting, level $\alpha$, and sample size $n$ equal to $1000$ or $3000$, we generate $100$ independent simulation replicates. In each replicate, the sample of size $n$ is randomly divided into a training sample and a test sample of equal size. The training sample is used to estimate the conditional mean function (or quantile regression), whereas the calibration sample is used to compute the residual scores, estimate the prediction radius, and select the corresponding tuning parameters. For the proposed heteroscedastic procedure, the neighborhood size used to estimate the local radius is selected according to the global and local empirical calibration criterion described in Algorithm \ref{algor:hete}.

\noindent In addition, for each replicate, we generate an independent evaluation sample of size 5000. This sample is used exclusively to evaluate the fitted prediction regions and is not used for model estimation, calibration, or tuning. For every observation in the evaluation sample, we record whether the corresponding response is contained in the estimated prediction region. We then fit a binomial generalized additive model \cite{wood2017generalized} to these coverage indicators as a smooth function of the covariate. More specifically, we estimate:

\[
p^{\alpha}(x)
=
\mathbb{P}\left\{
\mathbb{I}\left(Y \in \widetilde{C}^{\alpha}(X)\right) = 1
\mid X = x
\right\}.
\]

\noindent For each replicate \(b\), $b=1\dots,B$, we measure the \(L^2\) deviation from the nominal coverage \(1-\alpha\) as  
\[
\widetilde{\mathrm{ISCE}}^{\alpha}_{b}
= \int_{0}^{5} \bigl( \widetilde{p}^{\alpha}_{b}(x) - (1-\alpha) \bigr)^2 \,\mathrm{d}x,
\]  
and then average over replicates: \(\frac{1}{B}\sum_{b=1}^B \widetilde{\mathrm{ISCE}}^{\alpha}_{b}\).

\noindent To assess robustness, we use worst-slice coverage. For a prediction set \(\widetilde{C}^{\alpha}(X)\) and a class \(\mathcal{A}\) of covariate subsets,
\[
\mathrm{WSC}(\widetilde{C}^{\alpha};\mathcal{A}) = \inf_{\substack{A\in\mathcal{A}:\\ \mathbb{P}(X\in A)>0}} \mathbb{P}\bigl( Y \in \widetilde{C}^{\alpha}(X) \mid X \in A \bigr).
\]
The \(\delta\)-restricted version \(\mathrm{WSC}_{\delta}\) replaces \(\mathbb{P}(X\in A)>0\) with \(\mathbb{P}(X\in A)\ge \delta\) (we set \(\delta=0.05\) in simulations). We take \(\mathcal{A}\) as projection-induced slices: for each direction \(u\in\mathbb{S}^{p-1}\) and interval \(I\in\mathcal{I}\) (a prespecified family in \(\mathbb{R}\)),
\[
A_{u,I} = \{ x\in\mathbb{R}^{p} : \langle u,x\rangle \in I \},\qquad \mathcal{A} = \{ A_{u,I} : u\in\mathbb{S}^{p-1},\, I\in\mathcal{I} \}.
\]
Thus \(\mathrm{WSC}\) is the minimal conditional coverage over these slices, while \(\mathrm{WSC}_{\delta}\) ignores slices with mass below \(\delta\). To estimate local undercoverage, we use an independent evaluation sample and a fixed collection \(\mathcal{A}_{100}\) of 100 covariate intervals each of probability mass of at least \(0.05\):
\[
\widehat{\mathrm{WSC}} = \min_{A \in \mathcal{A}_{100}} \frac{ \sum_{i=1}^{n_{\mathrm{eval}}} \mathbf{1}\{X_i^{\mathrm{eval}} \in A\} \mathbf{1}\{Y_i^{\mathrm{eval}} \in \widetilde{C}^{\alpha}(X_i^{\mathrm{eval}})\} }{ \sum_{i=1}^{n_{\mathrm{eval}}} \mathbf{1}\{X_i^{\mathrm{eval}} \in A\} }.
\]
\noindent Because $X$ is scalar in all simulations presented in the main article, no projection directions or directional optimization are needed. These steps are only required for some examples in the Supplemental Material.

\paragraph{Simulation Settings}
\noindent

\noindent In our simulation study, we generate data from four scalar-to-scalar regression
models designed to cover linear and nonlinear mean structures, as well as
homoscedastic and heteroscedastic noise patterns. 
Our procedure uses $k$NN both for estimating the conditional Fréchet mean and for estimating the local prediction radius. For estimating the
conditional Fréchet mean, the tuning parameter \(k\) is selected by leave-one-out
cross-validation, following the methodology of \cite{azadkia2020optimal}. The
value of \(k\) used for radius estimation on the second split is then selected
according to the general calibration criterion introduced in the Algorithm \ref{algor:hete}.

\noindent

\paragraph{Setting 1 (Linear, heteroscedastic).}
\[
  X \sim \mathcal{U}(0,5),\quad
  \varepsilon \sim \mathcal{U}(-1,1),\quad
  Y = 3 + X + X\,\varepsilon.
\]

\paragraph{Setting 2 (Non‐linear, heteroscedastic; uniform noise).}
\[
  X \sim \mathcal{U}(0,5),\quad
  \varepsilon \sim \mathcal{U}(-1,1),\quad
  Y = 3 + e^{X} + X\,\varepsilon.
\]

\paragraph{Setting 3 (Non‐linear, heteroscedastic; Gaussian noise).}
\[
  X \sim \mathcal{U}(0,5),\quad
  \varepsilon \sim \mathcal{N}(0,4),\quad
  Y = 3 + e^{X} + X\,\varepsilon.
\]

\paragraph{Setting 4 (Additive, homoscedastic).}
\[
  X\overset{\mathrm{}}{\sim} \mathcal{U}(0,5),\quad \varepsilon \overset{\mathrm{}}{\sim} \mathcal{U}(-2.5,2.5) \quad
  Y = X + \varepsilon.
\]

\paragraph{Results}

Table~\ref{tab:results1} reports the mean \(L^2\) prediction error and worst-slice coverage for \(\alpha \in \{0.50, 0.20, 0.05\}\) and sample sizes \(n \in \{1000, 3000\}\). Our method constructs prediction sets in only a few seconds on a standard laptop, whereas the depth-profile approaches (implemented in \textsc{C++}) require about 50 minutes when \(n = 3000\). Average running times for all methods are given in the Supplemental Material. In the Supplemental Material, we also provide boxplots of the error distribution to illustrate estimation variability, not just average performance. Across the four simulation settings, no method uniformly dominates all competitors. The proposed heteroscedastic procedure performs particularly well in the nonlinear heteroscedastic settings, achieving the lowest ISCE and the highest or nearly highest WSC in Settings~2 and~3. CQR is the strongest competitor in Setting~1 and also performs well in the homoscedastic Setting~4, where the proposed global conformal procedure is highly competitive. These results indicate that local radius estimation is especially beneficial under nonlinear heteroscedasticity, whereas global or scalar-specific methods remain competitive when the data-generating mechanism is homoscedastic or comparatively simple.


\subsection{Graphical Analysis of Conditional Coverage in the Support of Random Variable $X$}

We next examine how conditional coverage varies over the support of $X$. The Supplemental figures show that the proposed heteroscedastic procedure provides the most stable coverage in the nonlinear heteroscedastic settings, particularly in Settings~2 and~3. In Setting~1, CQR remains highly competitive, while in the homoscedastic Setting~4 the global conformal procedure and CQR perform as well as, or slightly better than, the local method. These findings are consistent with the expected trade-off between local adaptation and statistical stability: local radius estimation is most beneficial when the conditional distribution changes substantially across the covariate space.

\subsection{Conclusions from the Simulation Results in the Supplemental Material}

\noindent In Supplemental Material~E, we provide additional simulations involving more complex outcomes, including probability distributions, Laplacian graphs, and multivariate Euclidean data. The main conclusions are as follows. First, in homoscedastic signal-to-noise regimes, the heteroscedastic algorithm does not satisfy finite-sample conformal guarantees because it is not conformally calibrated. Nevertheless, its empirical performance in terms of marginal coverage and worst-slice coverage is similar to that of the homoscedastic method, which is theoretically valid and close to optimal under homoscedasticity. Second, in heteroscedastic signal-to-noise regimes, the proposed adaptive method achieves good marginal calibration and worst-slice coverage close to the nominal level, whereas the homoscedastic method performs poorly in terms of worst-slice coverage. In small-sample regimes, however, worst-slice coverage for adaptive methods can be less stable. From a statistical learning perspective, this behavior is expected, since conditional calibration is impossible without either imposing additional generative assumptions or departing from the standard distribution-free conformal prediction framework.

\noindent Third, the performance of the $k$NN procedure can deteriorate as the dimension of the predictors increases, as is common for many algorithms. Its performance is also sensitive to the choice of $k$. Overall, the heteroscedastic $k$NN method achieves good marginal calibration despite not being conformal, and it provides good conditional calibration when the sample size is sufficiently large, without introducing explicit smoothing assumptions in the calibration step. These empirical findings are consistent with the scalar-on-scalar benchmark, where the proposed methods outperform several established approaches for scalar responses.

	\section{Analysis of real-world clinical data}

   In this section, we illustrate the applicability of the proposed framework for clinical outcomes taking values in different metric spaces. We first consider an unconditional shape-analysis problem involving handwritten digital biomarkers from individuals with and without Parkinson’s disease. We then examine a regression setting with distribution-valued continuous glucose monitoring outcomes, using global Fréchet regression to estimate the conditional center. Additional applications are presented in the Supplemental Material, including a time-series analysis of probability distributions derived from continuous glucose monitoring data collected in a clinical trial.

\subsection{Shape Analysis}

\subsubsection{Motivation}

Shape analysis provides a natural framework for studying random objects through their geometric characteristics and has important applications in biomedical research; see, for example, \citet{dryden2016statistical} and \citet{steyer2023elastic}. Here, we consider handwriting trajectories recorded from healthy controls and individuals with Parkinson's disease. Our objective is to construct Fr\'echet-centered reference regions that describe the typical variation in handwriting shapes among controls and to examine whether trajectories from individuals with Parkinson's disease tend to fall outside these regions. The analysis is intended as an exploratory illustration of the proposed framework for shape-valued observations, rather than as a validated classification procedure.

\subsubsection{Data Description}

We analyze the handwriting dataset of \citet{isenkul2014improved}, which contains digital trajectories from $n_{\mathrm{control}}=15$ healthy controls and $n_{\mathrm{PD}}=30$ individuals diagnosed with Parkinson's disease. Because no covariates are included, the analysis is based on the unconditional Fr\'echet mean of the control trajectories. We estimate the empirical Fr\'echet mean using the Euclidean metric $
d_1(\cdot,\cdot) = \|\cdot-\cdot\|_2.$ Prediction radii are then defined using the supremum metric $d_2(\cdot,\cdot) = \|\cdot-\cdot\|_\infty$,  which provides an interpretable measure of the maximum deviation between two handwriting trajectories. For a nominal coverage level $1-\alpha$, the corresponding empirical reference region is

\[
\widetilde C^\alpha = \mathcal{B}\!\left( \widetilde m, \widetilde q_{1-\alpha} \right),
\]
\noindent where $\widetilde m$ is the empirical Fr\'echet mean of the control group and $\widetilde q_{1-\alpha}$ is the empirical $(1-\alpha)$-quantile of the control residuals $d_2(Y_i,\widetilde m), i=1,\dots, 15$. The graphical envelopes shown below are used to visualize these metric-ball reference regions. Because the control sample is small, all control observations are used to estimate both the center and the radius. Consequently, these regions are descriptive empirical reference regions and do not carry the finite-sample split-conformal guarantee established for independent training and calibration samples.

\subsubsection{Results}

Figure~\ref{fig:firma} displays the handwriting trajectories together with empirical reference regions estimated from the control group. We consider miscoverage levels ($\alpha=0.5$) and ($\alpha=0.7$), corresponding to nominal coverage levels of (50\%) and (30\%), respectively. Trajectories from individuals with Parkinson's disease are highlighted in yellow. More than (75\%) of these trajectories lie outside the (50\%) reference region, suggesting that their geometric patterns frequently differ from those observed among the controls. However, because the control sample is small and the same observations are used to estimate both the Fréchet mean and the prediction radius, this finding should not be interpreted as validated diagnostic or classification performance. Rather, it provides an exploratory illustration of the potential value of shape-based digital biomarkers for characterizing Parkinson's disease patterns.

\begin{figure}[ht!]
    \centering
    \includegraphics[width=0.8\textwidth]{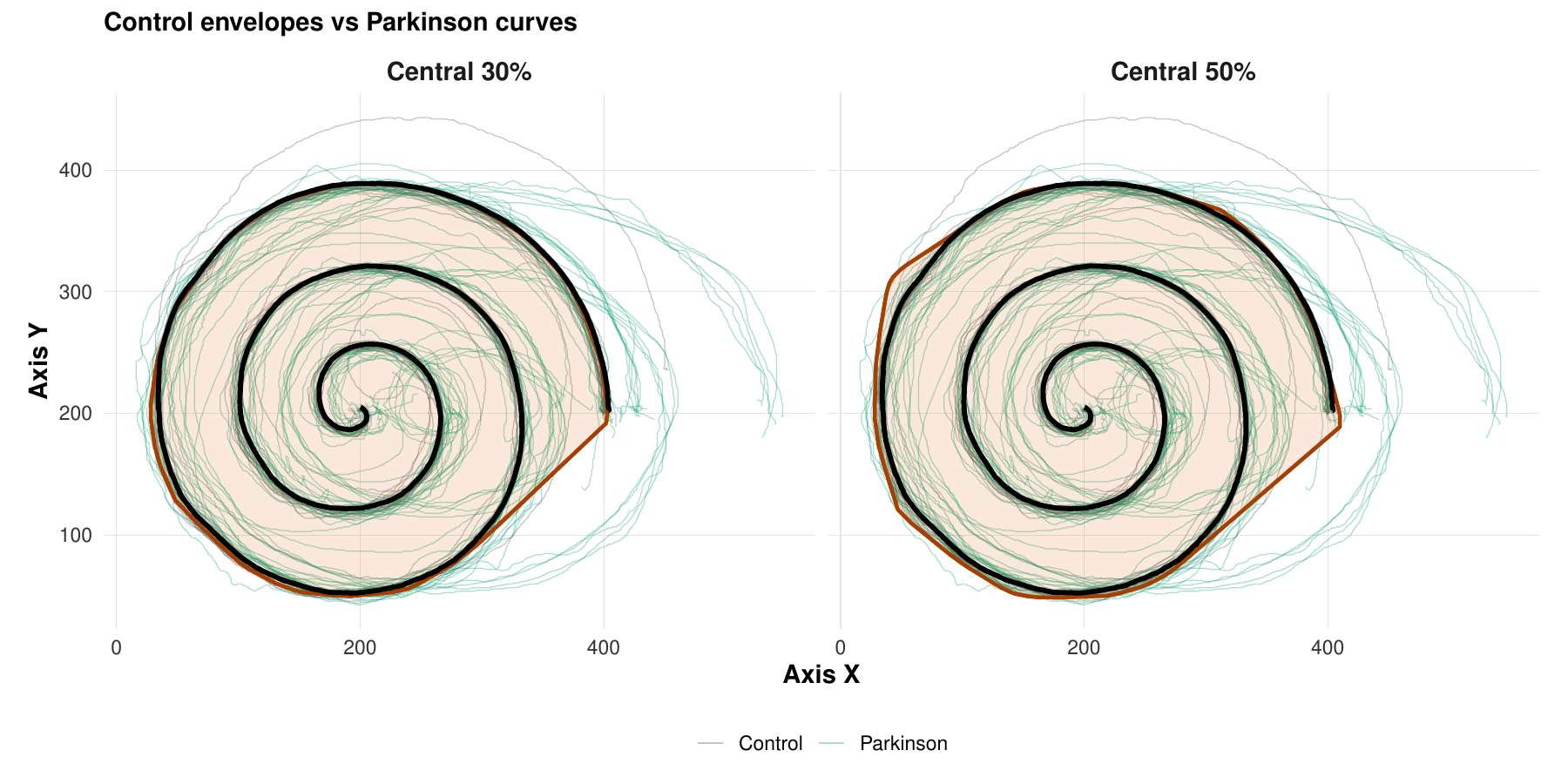}
\caption{Prediction regions for various confidence levels in the written test using control individuals as a reference. The plots depict the raw trajectories of both control and Parkinson's patients. We do not introduce covariates in the derivation of prediction regions for the Fréchet mean.}
    \label{fig:firma}
\end{figure}

\subsection{Distribution-Valued Representation of Biosensor Time Series}

\subsubsection{Motivation}

\begin{figure}[ht!]
\centering
\includegraphics[width=\linewidth]{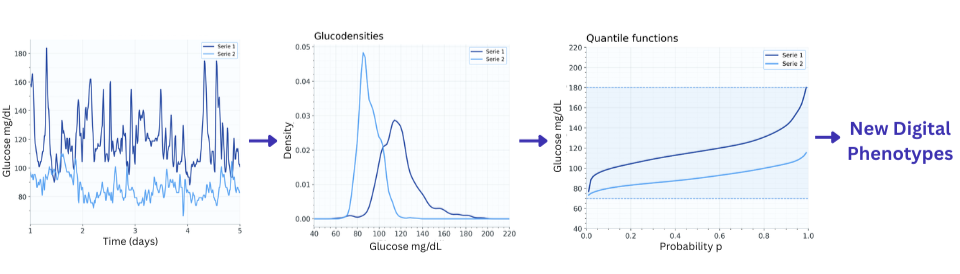}
\caption{Transformation of continuous glucose monitoring (CGM) time series from two individuals into their marginal density and quantile functions showing the new digital phenotypes considered in digital health.}
\label{fig:glucosetrans}
\end{figure}

Distributional data analysis has received increasing attention because many complex observations are more naturally represented by probability distributions than by finite-dimensional summaries \citep{klein2023distributional,matabuena2022kernel,pmlr-v162-meunier22b}. This representation is particularly useful for biological time series, whose marginal distributions may contain clinically relevant information that is not captured by conventional summaries such as the mean and standard deviation.

We consider continuous glucose monitoring (CGM) data, which provide repeated glucose measurements throughout the day and allow the full distribution of glucose values to be examined \citep{doi:10.1177/0962280221998064}. Distributional representations can capture behavior across low, intermediate, and high glucose concentrations, including patterns associated with hypoglycemia and hyperglycemia. Figure~\ref{fig:glucosetrans} illustrates the transformation of individual CGM time series into marginal density and quantile functions.

For each participant, we represent the CGM trajectory by its empirical probability distribution. The resulting response takes values in
\[
\mathcal{Y} = \mathcal{W}_2(T), \qquad T=[40,400]\subset\mathbb{R},
\]
where $\mathcal{W}_2(T)$ denotes the space of probability distributions supported on $T$, equipped with the $2$-Wasserstein metric. In the univariate setting, this metric is equivalent to the $L^2$ distance between quantile functions. This representation enables us to investigate how the entire glucose distribution varies with age, rather than restricting the analysis to a small number of CGM summary measures. Our objective is to estimate age-specific reference distributions and corresponding prediction regions among individuals without diabetes. These regions provide a descriptive representation of the range of glucose distributions expected at different ages and may serve as a basis for future research on personalized CGM reference values. 

\subsubsection{Data Description and Statistical Analysis}

We analyze data from \citet{shah2019continuous}, who collected glucose measurements at five-minute intervals over several days from 153 individuals without diabetes under free-living conditions (52 men and $101$ women; age range $7$-$80$ years). The original study was motivated by the need to characterize classical CGM metrics in healthy individuals using modern and more precise CGM technology. Using age-stratified summaries of glucose trajectories, the authors found that participants spent a median of 96\% of the observed time within the $70-140$ mg/dL range, with only a small proportion of measurements falling outside this interval. These results provide a reference characterization of glucose profiles across the age spectrum in a population without diabetes. Here, we take a different perspective by treating each participant's empirical glucose distribution as the response and age as the predictor. The conditional Fr\'echet mean is estimated using global Fr\'echet regression under the (2)-Wasserstein metric. Because univariate Wasserstein geometry can be represented through quantile functions, the estimated conditional Fr\'echet mean corresponds to an age-specific reference quantile function. To construct prediction regions, we measure residual variation using the supremum distance between quantile functions,
\[
d_2(\cdot,\cdot) = \|\cdot-\cdot\|_{\infty}.
\]
\noindent Thus, the $2$-Wasserstein metric is used to estimate the conditional center, whereas the supremum metric determines the radius and geometry of the prediction regions. We run the algorithm with $
\alpha = 0.20,$ corresponding to a nominal coverage level of $80\%$. The age-dependent prediction radius is estimated using the proposed $k$NN procedure, which selected \(k=20\).

\subsubsection{Results}

Figure~\ref{fig:glucoseref} presents the age-specific reference distributions and their corresponding ($1-\alpha=0.8$) coverage prediction region. The estimated conditional center remains relatively stable across the ages considered, whereas the prediction regions widen at older ages. This pattern suggests greater variability in the distribution of glucose values among older individuals, particularly in the lower and upper tails. The widening of the prediction regions should, however, be interpreted cautiously. It may reflect increased biological heterogeneity with age, a plausible interpretation consistent with the existing literature based on average glucose values \cite{keshet2023cgmap}. However, our analysis extends this perspective by considering the full marginal distribution of CGM time series rather than conventional scalar summaries alone. Overall, the analysis illustrates how distribution-valued prediction regions can characterize variation across the full glucose range and provide more detailed age-specific reference profiles than summaries such as mean glucose and standard deviation. Their potential use for screening or clinical decision-making would, however, require validation in independent cohorts including individuals both with and without diabetes.

\begin{figure}[H]
\centering
\begin{minipage}[b]{0.45\textwidth}
\centering
\includegraphics[width=\linewidth]{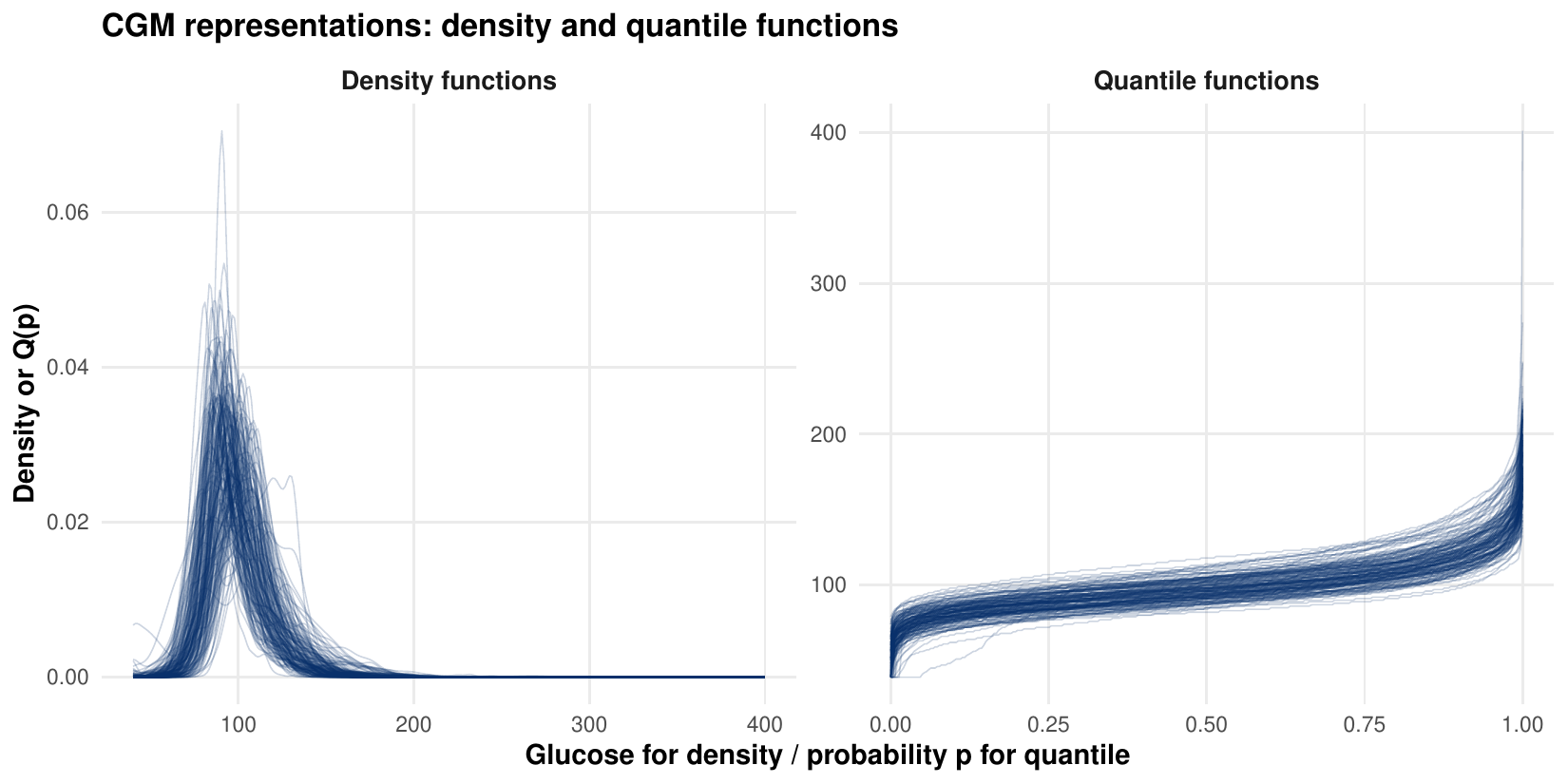}
\caption*{(a) Original representation of the CGM time series in terms of their density and quantile functions. }
\end{minipage}
\hfill
\begin{minipage}[b]{0.45\textwidth}
\centering
\includegraphics[width=\linewidth]{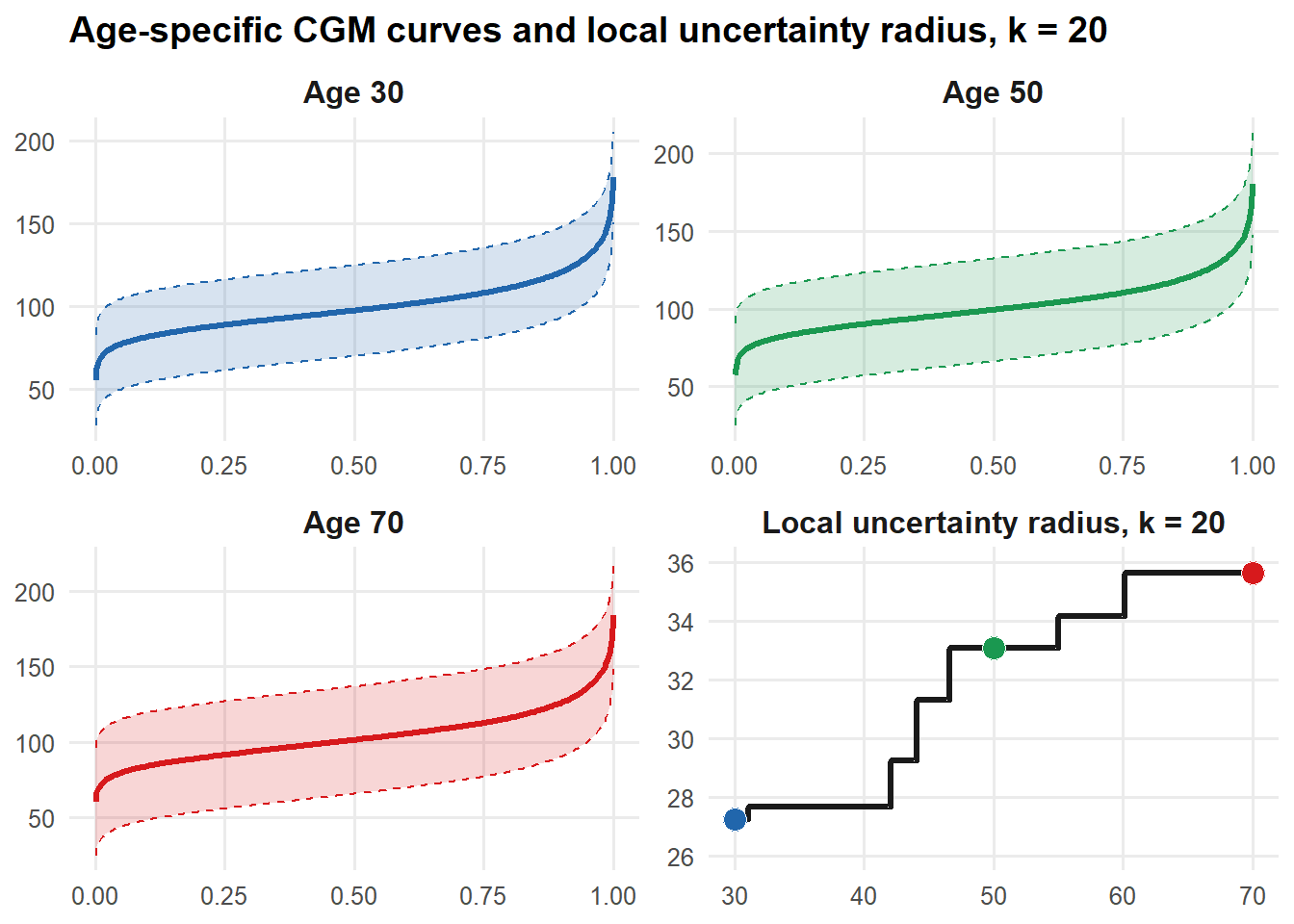}
\caption*{(b) Prediction regions for different ages and radius estimation.}
\end{minipage}
\caption{Estimated age-specific reference distributions for continuous glucose monitoring outcomes and their corresponding prediction regions at miscoverage level ($\alpha=0.20$), corresponding to ($80$\%) nominal coverage, for individuals aged 30, 50, and 70 years. The estimated conditional centers are relatively similar across the ages considered, whereas the prediction regions widen with age. The age-specific radii, estimated using ($k$=20) nearest-neighbors, increase with age.}

\label{fig:glucoseref}
\end{figure}



	\section{Discussion}

\noindent This paper introduces a general framework for constructing prediction regions for random objects in metric spaces. In the homoscedastic setting, we extend split conformal prediction using the distance between the response and an estimator of its conditional Fr\'echet mean as the conformity score. The prediction radius is then determined by a single unconditional scalar quantile, yielding finite-sample marginal coverage and fast convergence rates driven by the estimation errors of the Fr\'echet mean and the global quantile.

\noindent For heteroscedastic data, inspired by \citet{gyofi2020nearest}, we propose a local $k$NN method for estimating covariate-dependent prediction radii. The neighborhood size is selected through a data-driven criterion that balances global and local empirical calibration. The resulting center--radius construction is modular, computationally scalable, and compatible with any suitable regression estimator. Our theory establishes consistency conditionally on the observed sample and provides explicit bounds separating the errors due to center and radius estimation.

\noindent Compared with the distance-profile approach of \citet{zhou2024conformalinferencerandomobjects}, our method restricts prediction sets to metric balls, sacrificing some flexibility for multimodal or anisotropic distributions. In return, it produces simple and interpretable regions, avoids repeated local smoothing, and achieves substantially faster rates in the homoscedastic setting. The framework also extends to metric-space-valued time series through nearest-neighbor expert aggregation. Simulations and biomedical applications illustrate its accuracy and flexibility, although metric balls may remain conservative in strongly non-spherical settings. Future work will investigate more flexible prediction regions beyond metric balls by extending recent metric-learning ideas developed for Euclidean conformal prediction \citep{xu2024conformal,gauthier2026backward} to general metric spaces. We will also investigate extensions of causal inference methods to settings where the outcomes are random objects taking values in general metric spaces \citep{bhattacharjee2025doublyrobustestimationcausal,kurisu2025geodesic}.

	\bibliographystyle{plainnat}
	\bibliography{manuscript}
	\end{document}